\documentclass[preprint]{article}
\usepackage{neurips_2020}

\usepackage{times}

\usepackage{float}
\usepackage{microtype}
\usepackage{graphicx}
\usepackage{subcaption}
\usepackage{booktabs} 
\usepackage{amsmath}
\usepackage{amsthm}
\usepackage{amssymb}
\usepackage{natbib}
\usepackage{xcolor}
\usepackage{chngcntr}
\usepackage{apptools}
\AtAppendix{\counterwithin{thm}{section}}
\AtAppendix{\counterwithin{figure}{section}}

\usepackage{xcolor}
\usepackage[utf8]{inputenc}

\DeclareFixedFont{\ttb}{T1}{txtt}{bx}{n}{8} 
\DeclareFixedFont{\ttm}{T1}{txtt}{m}{n}{8}  

\usepackage{color}
\definecolor{deepblue}{rgb}{0,0,0.5}
\definecolor{deepred}{rgb}{0.6,0,0}
\definecolor{deepgreen}{rgb}{0,0.5,0}
\definecolor{deeporange}{rgb}{0.6,0.25,0}
\definecolor{verylightgray}{rgb}{0.97,0.97,0.97}
\definecolor{mydarkblue}{rgb}{0,0.08,0.45}
\definecolor{bluegray}{rgb}{0.3,0.3,0.6}

\usepackage[pdftex,
    colorlinks=true,
    linkcolor=black,       urlcolor=bluegray,
    anchorcolor=bluegray,  filecolor=bluegray,
    menucolor=bluegray,    citecolor=black
]{hyperref}

\usepackage{listings}

\newcommand\pythonstyle{\lstset{
language=Python,
basicstyle=\footnotesize\ttm,
keywordstyle=\ttb\color{deepblue},
commentstyle=\ttm\color{deepgreen},
stringstyle=\color{deeporange},
emphstyle=\ttb\color{deepred},    
emph={MyClass,__init__},          
otherkeywords={self,yield},       
frame=single,                     
showstringspaces=false, 
breaklines=true,
backgroundcolor=\color{verylightgray},
}}

\lstnewenvironment{python}[1][]
{
\pythonstyle
\lstset{#1}
}
{}

\newcommand\pythoninline[1]{{\pythonstyle\lstinline!#1!}}

\usepackage{thmtools}

\declaretheorem[name=Theorem,style=theorem]{thm}

\usepackage[utf8]{inputenc}

\DeclareFixedFont{\ttb}{T1}{txtt}{bx}{n}{8} 
\DeclareFixedFont{\ttm}{T1}{txtt}{m}{n}{8}  

\usepackage{listings}
\usepackage{hyperref}
\usepackage{cleveref}

\title{Automatic Differentiation Variational \\ Inference  with Mixtures}

\author{ {\bf Warren R. Morningstar\thanks{\quad Primary contact; \href{mailto:wmorning@google.com}{wmorning@google.com}.}} \\
Google Research \\
\And
{\bf Sharad M. Vikram}  \\
Google Research \\
\And
{\bf Cusuh Ham\thanks{\quad Work performed during an internship with Google Research.}} \\
Georgia Institute of Technology \\
Google Research \\
\And 
{\bf Andrew G. Gallagher} \\
Google Research \\
\And
{\bf Joshua V. Dillon\thanks{\quad Primary contact; \href{mailto:jvdillon@google.com}{jvdillon@google.com}.}} \\
Google Research \\
}

\newcommand{\E}{\mathbb{E}}
\newcommand{\KL}[2]{D_{\text{KL}}(#1, #2)}

\begin{document}

\maketitle

\begin{abstract}
    
Automatic Differentiation Variational Inference (ADVI) is a useful tool for efficiently learning probabilistic models in machine learning.  Traditionally,
approximate posteriors learned by ADVI are forced to be unimodal in order to facilitate use of the reparameterization trick. In this paper, we show how stratified sampling may be used to enable mixture distributions as the approximate posterior, and derive a new lower bound on the evidence analogous to the importance weighted autoencoder (IWAE).  We show that this ``SIWAE'' is a tighter bound than both IWAE and the traditional ELBO, both of which are special instances of this bound.  We verify empirically that the traditional ELBO objective disfavors the presence of multimodal posterior distributions and may therefore not be able to fully capture structure in the latent space.  Our experiments show that using the SIWAE objective allows the encoder to learn more complex distributions which contain multimodality, resulting in higher accuracy, better calibration, and improved generative model performance in the presence of incomplete, limited, or corrupted data.
\end{abstract}

\section{Introduction}
Variational inference has become a powerful tool for Bayesian modeling using deep neural networks, with successes including image generation \citep{Kingma:14}, classification \citep{Alemi:17}, uncertainty quantification \citep{Snoek:19} and outlier detection \citep{Bishop:94,Nalisnick:18}.  Much of the recent success in variational inference have been driven by the relative ease of fitting models using ADVI, where small numbers of samples can be used for individual forward passes through a model, and noisy but unbiased gradients can be determined using the reparameterization trick, allowing the use of backpropagation in training and enabling traditional stochastic gradient methods \citep{Rezende:14,Kingma:14}.  Currently, one major limitation of ADVI is that it is only possible if the posterior distribution is reparameterizable. This has to date forced ADVI methods to utilize a limited set of possible distributions.  While there have been developments in extending reparameterization to broader classes of distributions  \citep[e.g., gamma and beta distributions; ][]{Ruiz:16}, multimodal distributions have remained elusive.

This paper explores using ADVI with mixture posterior distributions.  Mixture distributions present an advantage over unimodal distributions due to their flexibility \citep{Bishop:98,West:93}.  The contributions of this paper are as follows:
\begin{enumerate}
    \item We propose the SIWAE, a new lower bound on the evidence for the specific case of a mixture variational posterior. When applicable, the SIWAE is tighter than the evidence or importance-weighted evidence lower bounds.
    \item We demonstrate using a toy problem that SIWAE is better suited to approximate a known posterior distribution than the traditional ELBO or the score function estimator.
    \item We empirically show that models trained using the traditional ELBO objective often fail to discover multimodality in the latent space even if mixtures are used for the posterior.  We also show that SIWAE allows models to more easily infer multimodality when it exists.
    \item We demonstrate that models trained with SIWAE achieve higher classification accuracy and better model calibration than ELBO using incomplete feature information.

\end{enumerate}

\section{Approach}\label{sec:approach}

Consider a simple latent variable model with a single observed data point $x$ and corresponding latent variable $z$ along with a prior distribution $r(z)$ and likelihood $p(x | z)$. In probabilistic modeling, we are interested in the posterior distribution
$p(z | x)$, but generally, computing the posterior analytically is intractable. Variational inference is a strategy that reframes Bayesian
inference as an optimization problem by first introducing a surrogate variational posterior $q_\phi(z | x)$, where $\phi$
are free parameters, and then maximizing the evidence lower bound (ELBO) with respect to $\phi$. The ELBO is defined as,
\begin{align}
    \mathcal{L}_{\text{ELBO}}(\phi) &\triangleq \E_{q_\phi(z | x)}\left[\log{p(x|z)}\right] - \KL{q_\phi(z | x)}{r(z)}
\end{align}
and is a lower bound on the marginal probability of the data $\log p(x)$ \citep{Jordan:99}. In ADVI, we aim to compute $\nabla_\phi \mathcal{L}(\phi)$, but computing the ELBO is analytically intractable.
Both terms in $\mathcal{L}(\phi)$ are expectations over $q_\phi(z | x)$, so we approximate the
gradient by first drawing samples from $q_\phi(z | x)$ and computing the gradient of a Monte-Carlo approximation
of the ELBO, i.e.,
for a single sample $z^\prime \sim q_\phi(z | x)$, we see that  
   $\mathcal{L}_{\text{ELBO}}(\phi) \approx \log{p(x|z^\prime)} - \log q_\phi(z^\prime | x) + \log r(z^\prime)$.
   
When computing the gradient, ADVI differentiates through the sampling procedure itself, utilizing the \emph{reparameterization trick} \citep{Kingma:14, Rezende:14}. The reparameterization trick expresses sampling a random variable $z$ from its distribution as a transformation of noise drawn from a base distribution $\epsilon \sim p(\epsilon)$, where the transformation is a deterministic function of the parameters of the sampling distribution $\phi$.
In ADVI, we are restricted to ``reparameterizable'' posterior distributions -- distributions whose sampling procedure can be expressed in this way.
Although there has been notable work in growing this class of distributions, such as in \citet{Figurnov:18} and \citet{Jankowiak:18}, the choice of posterior in ADVI remains limited.

In this paper, we consider mixture posteriors for ADVI, specifically mixtures whose component distributions are reparameterizable.
Mixture distributions are a powerful class of posteriors, as growing the number of components can make them arbitrarily expressive,
but are challenging to use as posteriors in ADVI as sampling from a mixture is not naively reparameterizable, due to the discrete categorical variable that is sampled to assign a data point to a mixture component.
As seen in \citep{Roeder:17}, \emph{stratified sampling} can address this issue. In stratified sampling, we compute expectations by
sampling evenly over component distribuions (``strata'') and averaging using the weights of each stratum. For a mixture distribution, the natural stratification is each of the mixture component distributions. Rather than initially drawing an assignment
and then drawing a sample from the corresponding component distribution, we
draw one sample from each component individually and compute
a weighted average over the samples.
Formally, for any continuous and differentiable function $f(z)$ and mixture distribution
$q(z) \triangleq \sum_{k = 1}^K \alpha_k q_k(z)$, where $\alpha_k$ are the mixture weights and $q_k(z)$ are the components,
we can compute the expectation $\mathbb{E}_{q(z)}{f(z)}$ as follows:
\begin{align}
    \mathbb{E}_{q(z)}{f(z)} &= \int{f(z)\left(\sum_{k=1}^{K}{\alpha_{k}q_k(z)}\right) dz}
    = \sum_{k=1}^{K}{\alpha_{k}\int{f(z)q_{k}(z)}dz} \nonumber
    = \sum_{k=1}^{K}\alpha_{k}\E_{q_k(z)}\left[f(z)\right]
\end{align}
By pulling the sum over the mixture components outside of the integral over $z$ and sampling from each of the $K$ mixture components, we are able to compute the expectation
using the reparameterization trick, so long as the component distributions from the mixture are themselves reparameterizable.
Returning to ADVI, when the posterior $q_\phi(z | x)$ is a mixture distribution with weights $\{\alpha_{k, \phi}(x)\}_{k = 1}^K$ and components $\{q_{k, \phi}(z | x)\}_{k = 1}^K$, we can compute the ``stratified ELBO,'' or SELBO:
\begin{align*}
    &\mathcal{L}_{\text{SELBO}}(\phi) \triangleq 
     \sum_{k = 1}^K \alpha_{k, \phi}(x) \E_{z_k \sim q_{k, \phi}(z | x)}\left[\log \frac{p(x|z_k)r(z_k)}{q_\phi(z_k | x)}\right]
\end{align*}
While SELBO is technically the same objective as the ELBO but specialized
to mixtures, we draw this distinction to imply that we are drawing $K$
reparameterizable samples to compute a differentiable, Monte-Carlo estimate
of the SELBO
whereas the traditional ELBO formulation implies we take a single sample
to compute a non-differentiable estimate.

\subsection{A tighter bound for mixture posteriors}
While the SELBO objective allows us to fit a mixture posterior using ADVI, 
it falls prey to the same issues that make fitting multimodal distributions
with the ELBO difficult, namely the ELBO's mode-seeking behavior. Furthermore this
mode-seeking behavior actively works against the goal of learning a multimodal posterior. Assuming a truly multimodal posterior, an under-expressive unimodal distribution fit using the ELBO would only learn one mode.
A multimodal distribution would achieve a better fit, but the nature of the
ELBO objective makes this difficult, too.
A multimodal posterior will have regions of high density separated by regions
of low density and these low density regions penalize exploration.
A multimodal distribution being fit to this posterior will be penalized for
generating samples in any of these areas of low-density, and therefore will
be discouraged from exploring the landscape for surrounding modes. 
We thus expect ADVI to collapse mixture components to learn a conservative approximate posterior, neglecting to explore other distinct modes which may also be able to explain the data. To mitigate this exploration penalty, we can use importance sampling.

An importance-weighted estimate of the log-likelihood first draws $T$ i.i.d. samples from the posterior $\{z_t\}_{t = 1}^T \sim q_\phi(z | x)$, computing a lower bound using the ratio of the densities of a sample under the joint distribution and posterior (i.e., importance weights) for each sample (called ``IWAE'' in \citet{Burda:16}):
\begin{align*}
\mathcal{L}_{\text{IWAE}}^T(\phi) \triangleq \E_{\{z_t \sim q_\phi(z | x)\}_{t = 1}^T} \left[\log \frac{1}{T} \sum_{t = 1}^T \frac{p(x | z_t) r(z_t)}{q_\phi(z_t | x)}\right]
\end{align*}

\citet{Burda:16} shows that if the importance weights are bounded, then as $T$ increases the IWAE grows tighter and approaches $\log p(x)$ as $T \rightarrow \infty$. Unlike the regular ELBO,
the posterior in the IWAE is less penalized for generating samples that are unlikely. Instead, unlikely
samples are weighted less and the learned posterior can have higher variance to better cover
the space. This is a desirable property to avoid component collapse when
fitting mixture distributions using ADVI since it enables components
to explore distinct modes.

Our main contribution is a novel importance-weighted estimator for the ELBO when using mixture posteriors. To incorporate importance sampling into the SELBO, we first draw $T$ samples from each of the mixture components, $\{z_{kt}\}_{k = 1, t = 1}^{K, T}$. We then
compute importance weights that are themselves weighted by the mixture weights, arriving at the ``stratified IWAE,'' or SIWAE:
\begin{align}
    &\mathcal{L}_{\text{SIWAE}}^T(\phi)
    \triangleq \E_{\{z_{kt} \sim q_{k, \phi}(z | x)\}_{k=1,t=1}^{K, T}} \Bigg[
    \log \frac{1}{T} \sum_{t = 1}^T \sum_{k = 1}^K \alpha_{k, \phi}(x) \frac{p(x|z_{kt})r(z_{kt})}{q_{\phi}(z_{kt}|x)} \Bigg]
\end{align}

By repeated application of Jensen's equality, we can
demonstrate that $\mathcal{L}_\text{SIWAE}^T$ is a valid lower bound
that is tighter than $\mathcal{L}_\text{IWAE}^T$ when $K > 1$ (see theorems and proofs in \Cref{sec:theorems}).
$\mathcal{L}_\text{SIWAE}$ is also equivalent to $\mathcal{L}_\text{IWAE}^T$ and $\mathcal{L}_\text{SELBO}$ under certain circumstances ($K = 1$ and $K = T = 1$, respectively). 
Because $\mathcal{L}_\text{IWAE}$ is tighter than $\mathcal{L}_\text{SELBO}$ even when $T = 1$, $\mathcal{L}_\text{SIWAE}$ is also tighter
than $\mathcal{L}_\text{SELBO}$. Furthermore the importance sampling step enables higher-variance posteriors, as it 
mitigates the penalty for low-likelihood samples. Consequently, the implicit posterior \citep{Cremer:17} (defined by importance sampling the learned posterior) can better capture different modes.
Furthermore, SELBO and SIWAE are both easy to implement and are simple augmentations of existing variational inference code. See \autoref{fig:siwae-code} for a code snippet in TensorFlow \citep{Abadi:16} which evaluates the SIWAE for a latent variable model.

\begin{figure}[t]
\begin{python}
def siwae(prior, likelihood, posterior, x, T):
  q = posterior(x)
  z = q.components_dist.sample(T)
  z = tf.transpose(z, perm=[2, 0, 1, 3])
  loss_n = tf.math.reduce_logsumexp(
    (- tf.math.log(T) + tf.math.log_softmax(mixture_dist.logits)[:, None, :]
     + prior.log_prior(z) + likelihood(z).log_prob(x) - q.log_prob(z)),
    axis=[0, 1])
  return tf.math.reduce_mean(loss_n, axis=0)
\end{python}
\setlength{\belowcaptionskip}{-0.5cm}
\caption{\href{http://tensorflow.org/probability/}{TF Probability} implementation of SIWAE loss for local latent variable models (e.g., VAE).}
\label{fig:siwae-code}
\end{figure}

\section{Related Work}
\citet{Salimans:13} and \citet{Kingma:14} show that sampling from a distribution can be reparameterized as a deterministic function of the parameters of the distribution and some auxiliary variable, thereby facilitating the propagation of gradients through the distribution. They also introduced the Variational Auto Encoder (VAE), which uses an amortized variational posterior for a deep generative model.  \citet{Burda:16} showed that the bound on the evidence could be tightened using importance sampling, and that the tightness of the bound was improved by the number of importance samples. \citet{Cremer:17} suggest that the IWAE can be viewed as fitting the traditional ELBO, but using a richer latent space distribution defined by the importance-reweighting of samples from the posterior, and further explore the functional forms of these implicit posteriors.

While our work explores mixtures for the variational posterior, others have studied the use of (trainable) mixtures for the prior.  \citet{Dilokthanakul:16, Johnson:16, Jiang:17} introduce a VAE which uses a learnable mixture of Gaussians to represent the prior distribution of a latent variable. Learning a mixture prior does not require differentiating through the prior's sampling procedure as VI samples are drawn from the posterior not prior. \citet{Dilokthanakul:16,Jiang:17} find that their models achieve competitive performance on unsupervised clustering, with the mixture components learning clusters that approximate the different classes present in the data. Similarly, \citet{Tomczak:17} use a mixture of Gaussians trained on learnable pseudo-inputs as the prior, which allows them to introduce greater flexibility in the latent space distribution. They find that their generative performance improves on a number of benchmarks using this procedure. While using a mixture distribution as a prior enables modeling global structure in the latent space, it does not explicitly model ambiguity or competing explanations for a single observation. The uses of mixture distributions for either the prior or posterior are orthogonal and complementary, and a mixture distribution in either part of the model is a valid option.

\citet{Domke:19} propose to use alternative sampling schemes (including stratified) from a uniform distribution defined over a state space, along with a coupling transformation to the latent space in order to design a sampling scheme which results in better coverage of the approximating posterior distribution. They also show that the divergence of this approximation from the true posterior is bounded by the looseness of the evidence bound.  

When using mixture distributions as the posterior, the typical strategy is to fix component weights \citet{oh2018modeling}, or by using a continuous relaxation (e.g., the concrete relaxation of the categorical distribution \citet{Poduval:20}).  \citet{Graves:16} proposes an algorithm that allows for gradients to be backpropagated through the mixture distribution when the component distribution have diagonal covariances by composing the sampling procedure as a recursion over the dimensions.  Our method only requires that the component distributions is subject to reparameterization, and therefore can be used with a wider class of distributions. Furthermore it does not require explicit specification of the gradient updates to be hard-coded, making it easy to integrate mixtures into existing models. \citet{Roeder:17} derives a pathwise gradient extension to the SELBO that lowers the variance of gradient estimates, but still suffers from the mode-seeking properties of the SELBO.


\section{Experimental Results}
In this section, we compare experimental results with both deterministic models containing no latent variable, and with those containing only a single component parameterizing the latent space distribution. In \Cref{sec:full_im_vae} and \Cref{sec:mnist-style} are additional experiments exploring extensions
to those in the main body of the paper.
\subsection{Toy Problem}

\begin{figure*}[t]
    \centering
    \begin{subfigure}[t]{0.2\textwidth}
    \includegraphics[width=\hsize]{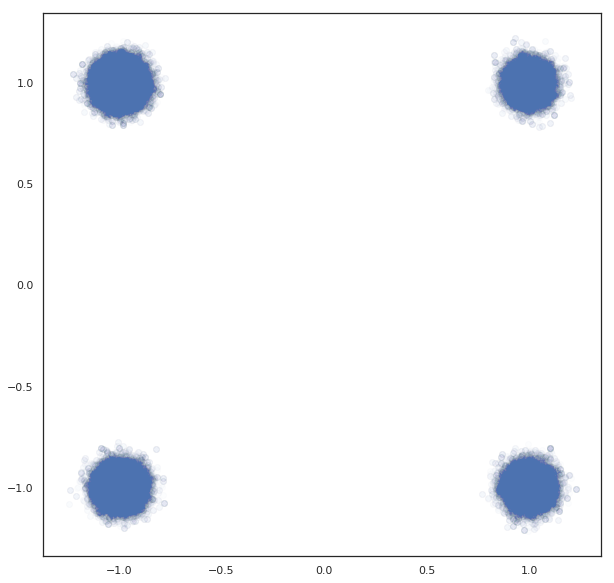}
    \caption{True posterior}
    \end{subfigure}%
    ~
    \begin{subfigure}[t]{0.2\textwidth}
    \includegraphics[width=\hsize]{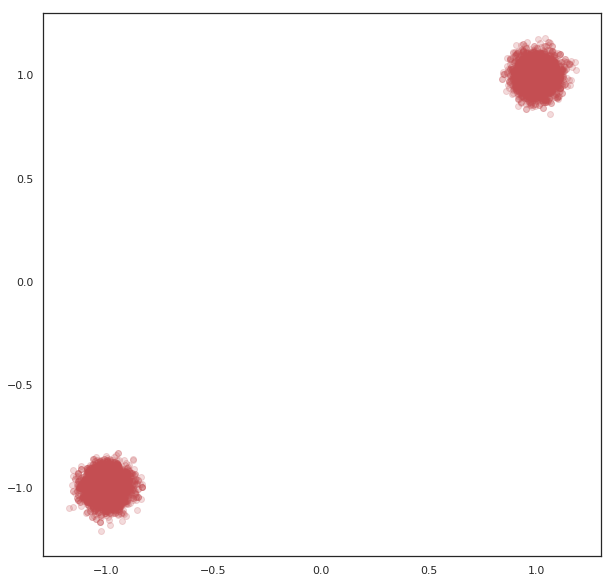}
    \caption{SELBO posterior}
    \end{subfigure}%
    ~
    \begin{subfigure}[t]{0.2\textwidth}
    \includegraphics[width=\hsize]{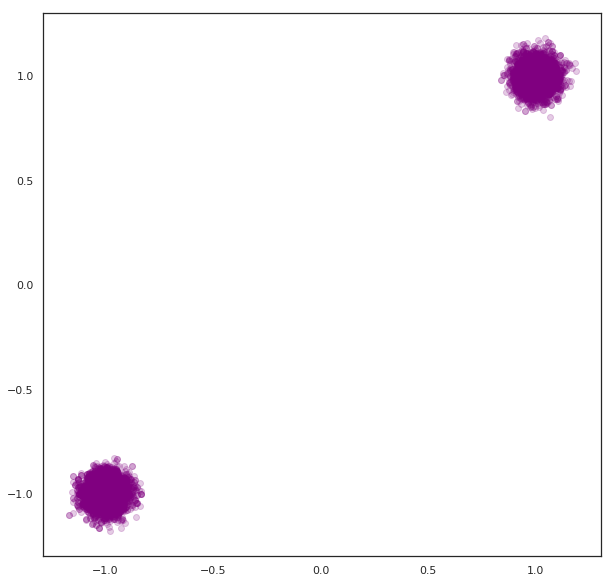}
    \caption{Score function posterior}
    \end{subfigure}%
    ~
    \begin{subfigure}[t]{0.2\textwidth}
    \includegraphics[width=\hsize]{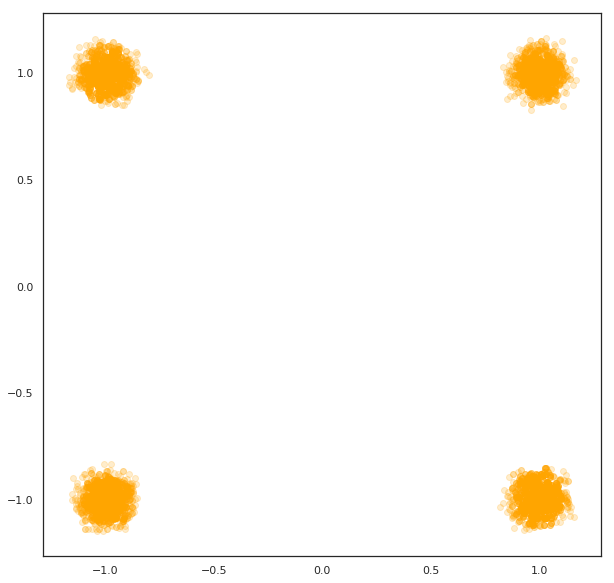}
    \caption{SIWAE posterior}
    \end{subfigure}%
    \setlength{\belowcaptionskip}{-0.5cm}
    \caption{We sample the true posterior along with each of the learned (implicit) posteriors for the observed data point $(1, 1)$. We see that the SELBO- and score function-trained posteriors are unable to capture
    all 4 modes of the true posterior.}
    \label{fig:toy-posteriors}
\end{figure*}

We define a latent variable model where the true posterior is multimodal by construction, with the hope
of recovering the distinct modes.
Specifically, we sample 1000 datapoints from the following two-dimensional generative model:
\begin{equation}
z \sim \mathcal{N}(0, I) \quad\quad\quad
x \sim \mathcal{N}\left(|z|, \sigma^2I\right) \nonumber
\end{equation}
where $\sigma^2 = 5\mathrm{e}{-2}$, i.e., we first sample a latent $z$ from an isotropic normal, but
observe $|z|$ with some Gaussian noise. For an observed $x$, 
there are $4$ distinct modes in $z$-space that could have generated it, since $z$ is two-dimensional.
We initialize the variational posterior
$q_\phi(z | x)$ as a multilayer perceptron (MLP) with 2 layers of 100 hidden units that outputs a 4-component mixture of Gaussians
distribution. We evaluate three different estimators of the ELBO: (1) SELBO, (2) SIWAE, and (3) a score function estimator as a baseline. We fit the posterior for 1000 epochs, with a batch size of 32 and using the
Adam \citep{Adam} optimizer with a learning rate of 0.001, using 10 importance samples for SIWAE and 100 for both SELBO and score function. Each baseline
was initialized and trained identically (same initial weights and order of batches).

We measure performance using a $10^{6}$-sample SIWAE estimate, and observe that the SIWAE-trained estimator achieves the highest value of -1.505, compared to -2.024 and -2.038 from the SELBO and score function estimators, respectively.
Investigating further, we plot samples from each of the implicit importance-weighted posteriors in the latent space. We find that in many cases, the SELBO and score function posteriors are unable to capture
the four distinct modes (see \autoref{fig:toy-posteriors}), whereas the higher-variance SIWAE posterior is able to cover
the modes successfully. We also observe similar results to those found in \citet{Rainforth:18},
where tighter variational bounds result in lower signal-to-noise ratios in the gradients
to the posterior. This is reflected by on-average higher-variance gradients while training
a SIWAE posterior vs. a SELBO posterior (1.16 vs. 0.48 average elementwise variance, respectively). However, the score function estimator has significantly higher empirical variance (261.4) than that of both SIWAE and SELBO, indicating that the variance
reduction coming from the use of the reparameterization trick offsets the additional variance from a tighter
variational bound. We also found that using the ``sticking-the-landing'' (stl) estimator \citep{Roeder:17}
(\autoref{fig:toy-dataset-stl}, \autoref{fig:toy-posteriors-stl}) 
does not significantly improve the SELBO or SIWAE in the toy experiment.

\subsection{Single Column MNIST Classification}\label{sec:VIB}

\begin{figure*}[ht]
    \centering
    \includegraphics[width=\hsize]{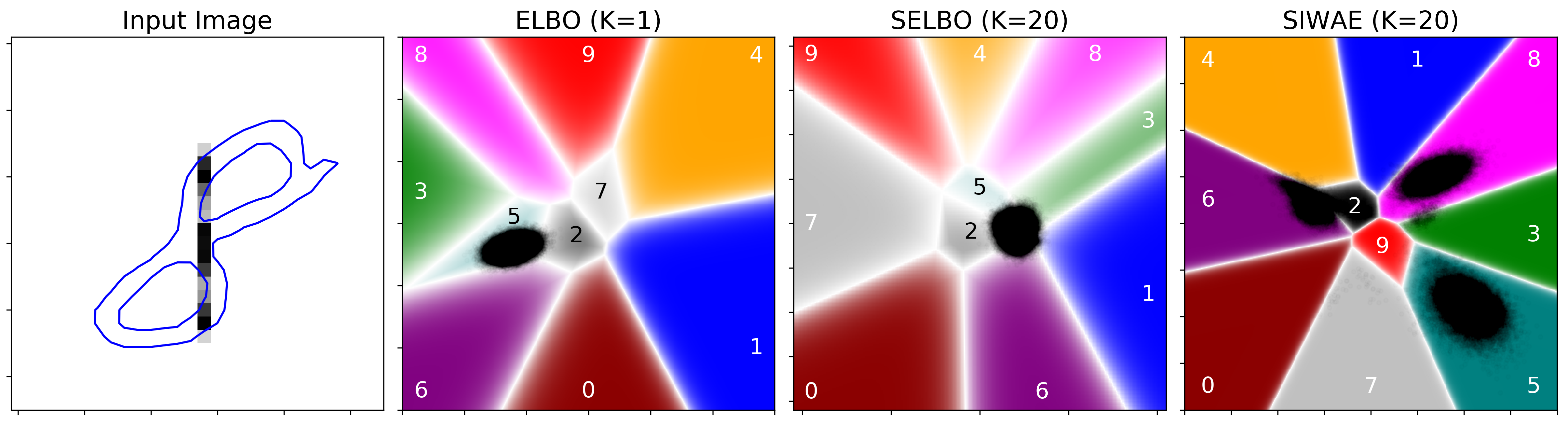}
    \caption{An illustration of the difference between the latent spaces learned by mixtures versus those found using unimodal variational posterior distributions.  The left panel shows the input to the network, for which all but the center-most column has been discarded.  The contour shows the true image of the data.  The second column shows a model trained using a unimodal latent space distribution and optimizing the ELBO.  The third column shows the latent space learned using the SELBO objective with 20 mixture components.  On the right, we show the latent space for the same example found using the SIWAE objective.  The latent space is colored by the predicted class from that position in the latent space, and the transparency of that color indicates the confidence of the predicted class relative to the second most probable class.} 
    \label{fig:local_elbo_latent}
\end{figure*}

\begin{figure*}
    \includegraphics[width=\hsize]{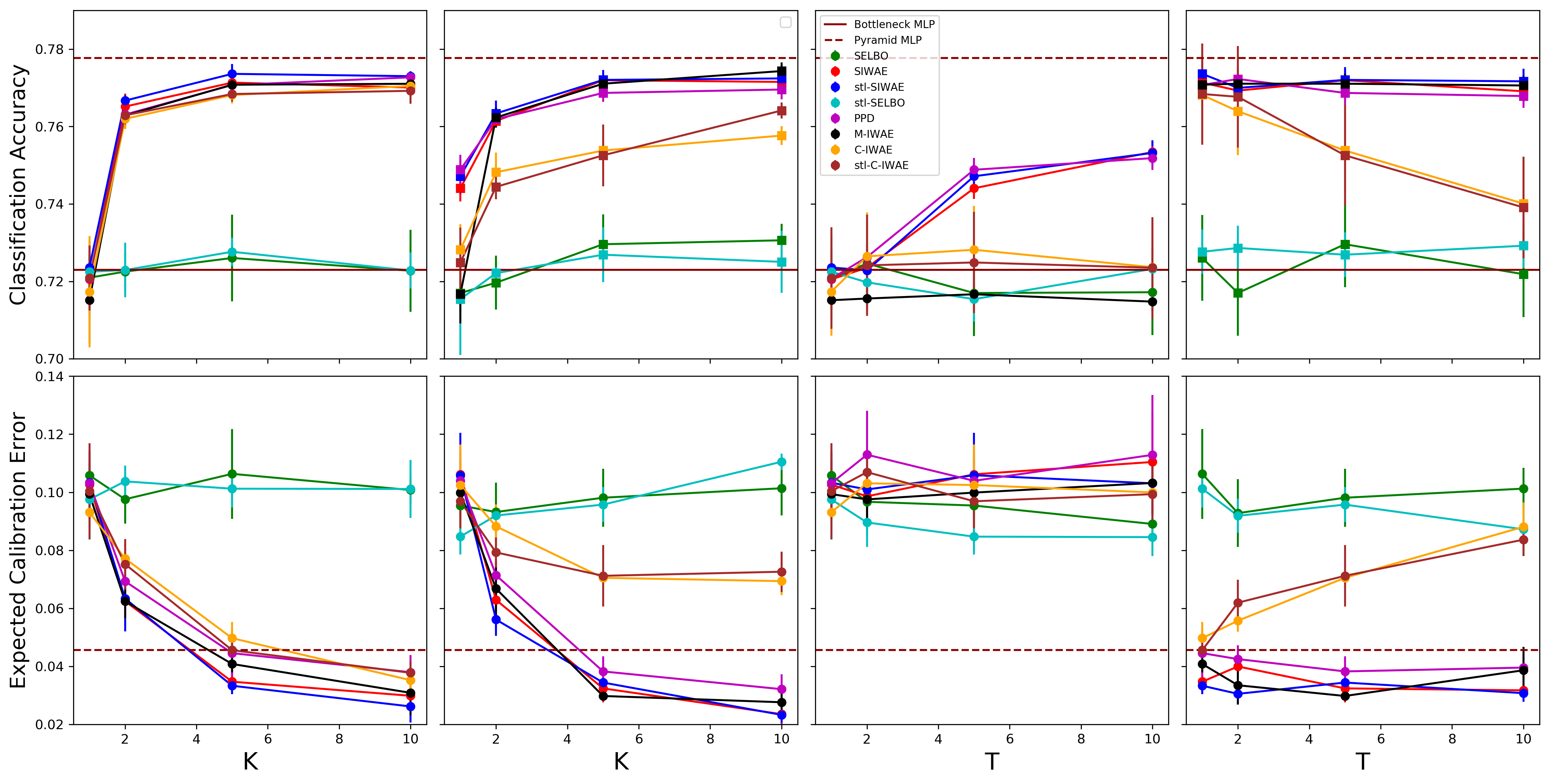}
    \caption{\textbf{Top row:} Classification accuracy of a model trained using the SELBO and SIWAE objectives as we vary the number of mixture components ($K$, left two panels) and the number of samples per component ($T$, right two panels).  \textbf{Bottom row:} Expected Calibration Error as a function of the number of mixture components $K$, or samples $T$.}\label{fig:ece}
\end{figure*}
\vspace{-5pt}

To evaluate SIWAE's efficacy on a more challenging problem, 
we trained a classifier on the benchmark dataset MNIST \citep{Lecun:98}.
The classifier is a Variational Information Bottleneck (VIB) model \citep{Alemi:17}, a variant of the VAE where the decoder outputs a class rather than a reconstructed input. To better motivate the use of mixtures on this dataset, we consider the problem of classification under incomplete information.  In particular, \citet{Doersch:16} shows that training a VAE using only the centermost column of the image introduces multimodality into the dataset that is difficult to capture using a unimodal encoder. We replicate this multimodality in the classification setting by taking the centermost column of each training image.  An example of a corrupted input can be seen in \autoref{fig:local_elbo_latent}. In general, it can be difficult even for a human to correctly classify the image given this type of corruption. In this scenario, we look for not only accurate predictions but also well-calibrated uncertainty for those predictions. 

For this set of experiments, we consider how classification performance varies as the number of mixture components are varied.  For this, we train models using $K=[1, 2, 5, 10]$ for the number of mixture components.  We also use $T=[1, 2, 5, 10]$ for the number of samples drawn \emph{per component}. For a single component model, we optimize both the traditional evidence lower bound (ELBO), as well as the importance weighted estimate of the evidence (IWAE). For the mixture models, we use stratified sampling to compute the ELBO (SELBO), as well as the Stratified-IWAE (SIWAE) derived in Section~\ref{sec:approach}.  In addition to comparing IWAE, ELBO, SIWAE, and SELBO, we also compare these models to other models trained using several additional losses suggested by previous works to potentially improve on IWAE-like models.  In particular, we trained models using the C-IWAE and M-IWAE from \citet{Rainforth:18}, which reduce the gradient variance and purportedly result in better models than those trained with IWAE.  For both of these, we use an appropriately modified SIWAE for $K>1$. We also train models using the stl gradient estimator from \citet{Roeder:17} which was also suggested to improve variance in the gradient estimates.  We added the stl estimator to SELBO, SIWAE, and C-IWAE to examine if the stl estimator has an effect on performance compared to models which use the naive implementations of the gradient.  Finally, we compare to an additional loss which maximizes the log-probability from the posterior predictive distribution:  $\mathcal{L}_{PPD} = -\log{( 1/TK \sum_{t, k=1}^{T, K}p(x|z))} + \KL{q(z|x),r(z)}$.  Additional details about the training procedure can be found in \Cref{sec:experiments_details}.

To evaluate the accuracy of the model, we first compute the predictive distribution by decoding $10^{4}$ samples from $q_{\phi}(z|x)$ and averaging the class probabilities returned by each sample. This marginalizes over the uncertainty in the latent variables and if our prior beliefs are correct, nominally produces calibrated probabilities. The predicted class is the one with the largest probability under the predictive distribution,
and accuracy of these predictions is measured on the test set.

In the middle two columns of \autoref{fig:local_elbo_latent}, we visualize samples from the posterior of a single validation set example learned by optimizing the ELBO/SELBO objective.  We find that, while SELBO enables the use of multiple mixture components in the variational posterior distribution, the model only learns a unimodal representation of the latent variable.  This is a direct consequence of the ELBO objective, which disincentivizes exploration and encourages mode-seeking in the variational posterior.  In this case, we observe the posterior ``hedging its bets,'' where the single mode sits across several decision boundaries.  These decision boundaries are also quite wide, suggesting that the model is using variance in the decoder as a source of uncertainty. We find this behavior undesirable, and show later that it negatively affects how well calibrated the model is.

The rightmost column of \autoref{fig:local_elbo_latent} shows the latent space learned by optimizing the SIWAE objective.  In stark contrast to models trained with SELBO, we find that SIWAE learns posteriors that have many active and distinct modes.  This implies that rather than ```hedging its bets'' as in the SELBO, a SIWAE-trained posterior offers multiple competing explanations, moving the uncertainty in the final prediction into the latent space rather than the output of the decoder.  This can be directly seen by looking at the lightness of the background colors in \autoref{fig:local_elbo_latent}, which indicate the confidence in the decoder prediction (less transparent, more saturated colors indicate more confidence in a prediction and vice-versa). Where the SELBO-trained decoder tends to have fuzzier, more transparent decision regions, the SIWAE-trained decoder has sharper, more confident decision boundaries. We later see how this property is critical for well-calibrated predictions.  Furthermore, while it is difficult to evaluate the interpretability of the latent space quantitatively, the SIWAE models are qualitatively easier to interpret using the latent space, with the model very clearly predicting the example shown as either a 5, an 8, or a 6 (with some additional limited probability that it is a 3). This appears to reflect our own intuition of the output class of this example.

We further investigate the quantitative model performance achieved by training with each objective as the number of mixture components and number of samples are varied.  \autoref{fig:ece} shows the classification accuracy of a VIB model over a range of $K$ and $T$.  From these results we make several observations:  1)  We find that SELBO-like losses disfavor multimodality (as seen in \autoref{fig:local_elbo_latent}), and therefore offer no improvement with additional mixture components or samples.  
2) We find that SIWAE-like losses overcome these deficiencies and therefore offer increased accuracy with additional mixture components (and samples for $K=1$).  For large $K$, $T$, the performance approaches the deterministic baseline, but does so using far fewer parameters.  3)  We find that the ``sticking the landing'' gradient estimator offers little benefit, since the accuracy returned from all corresponding naive implementations of each loss have comparable performance.   4)  We find that C-IWAE and stl-C-IWAE do not offer improvement over SIWAE, performing equally when $T=1$, and worse otherwise.  5)  We find that PPD generally performs comparably to SIWAE and stl-SIWAE, possibly because it also uses the latent space to codify uncertainty.  6) We find that M-IWAE exhibits similar performance to SIWAE as a function of $K$, but observes no improvement with $T$.  This behavior is expected, as M-IWAE is equivalent to SIWAE over $K$, and SELBO over $T$.  While we find the largest improvement results from using the SIWAE objective, it is important to note that arbitrarily growing the number of importance samples may also be harmful, a phenomenon observed by \citet{Rainforth:18}. We do not see any evidence for this over the range of $T = 1 \rightarrow 10$ importance samples, suggesting that positive effect of importance sampling enabling fitting better mixture models outweighs the negative effect of worse gradients.  However we also speculate that since the gradient variance scales as $T^{0.5}$, the performance may turn over for sufficiently large $T$.

While our SIWAE models with large $K$ appear to achieve comparable accuracy to the deterministic baseline, it is also important to compare their calibration, since real-world decision-making systems not only require accurate models, but also ones which quantify their uncertainty correctly. To measure the calibration objectively, we use the \emph{Expected Calibration Error} (ECE) from \citet{Guo:17} both for the deterministic baseline, as well as for the models trained with SELBO and SIWAE.  For the SELBO and SIWAE models, we use the posterior predictive distribution to evaluate the ECE, marginalized over 10000 samples from $q_{\phi}(z|x)$. We show the ECE as a function of the number of mixture components for our models in the bottom half of (\autoref{fig:ece}). In addition to the improvements in accuracy noted above, we find that increasing $K$ also decreases the calibration error when training using SIWAE-like objectives, outperforming the deterministic baseline for $K>2$, where the SIWAE objective should maximize a better approximation of the evidence.  We also find that models trained with SELBO do not have a corresponding improvement in calibration, with a final ECE more than 5 times larger than the equivalent model trained with SIWAE.  Additionally, models trained with SELBO never surpass the deterministic baseline of 4.5\% calibration error, in contrast to SIWAE models which achieve 2\% ECE for $K=10$.  This confirms our hypothesis that having predictive uncertainty stem from the decoder worsens the model calibration and that having multiple competing explanations in the latent space results in better calibrated predictions. See \Cref{sec:mnist-1d} for results in one dimensional latent space.

\vspace{-5pt}
\subsection{Single Column MNIST VAE}\label{sec:single_column_vae}
The SIWAE objective appears to successfully infer latent structure indicative of class boundaries using only a single column of the image.  However, a different and equally intriguing question is if this representation is also sufficient to reconstruct the image itself.  This question was explored by \citet{Doersch:16}, who showed that a class-conditional VAE was necessary to break the class degeneracy that can exist when the images are a single column.  Our hypothesis was that the use of a mixture posterior distribution can replicate this conditionality, without using the class labels.  

\begin{figure*}
    \centering
    \includegraphics[width=\hsize]{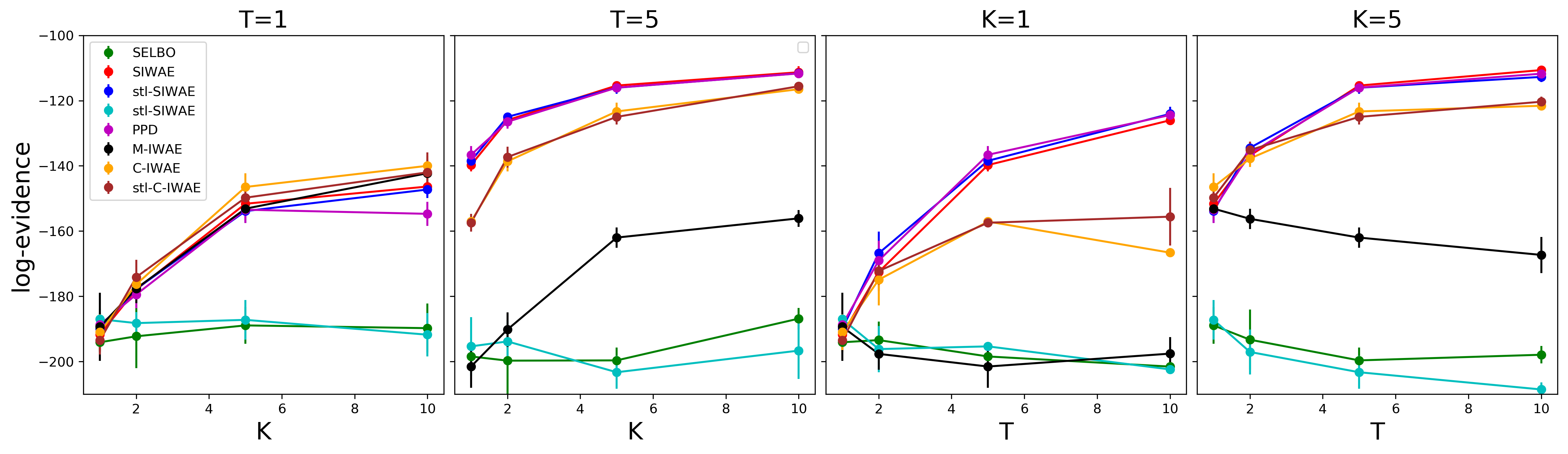}
    \setlength{\belowcaptionskip}{-0.5cm}
    \caption{Model evidence as a function of the number of mixture components $K$ or the number of samples per component $T$.  The evidence was measured using a single SIWAE estimate with 100 samples.  We find that models trained with SELBO-like losses appear to offer little to no noticeable improvement with either $K$ or $T$, while SIWAE-like losses offer substantial improvements with both.}
    \label{fig:singlecolumnevidence}
\end{figure*}

Our test setup is the same:  train a model with either the SIWAE or the SELBO loss, and observe performance as a function of $K$ and $T$.  This time, we use the log-evidence to measure performance, computed with a SIWAE estimate using 100 samples from the surrogate posterior.  We thought this was the most fair comparison, as it holds the total sample number fixed, and therefore highlights the difference based solely on the posterior expressiveness.  As with our classification experiments, we compare against stl-SIWAE, stl-SELBO, C-IWAE, stl-C-IWAE, M-IWAE, and PPD.

\Autoref{fig:singlecolumnevidence} shows the model evidence as a function of $K$ and $T$.  We find that for SIWAE-like trained models, the log evidence increases substantially with increasing $K$, indicative of the model successfully leveraging representational multimodality.  For SELBO-like losses, we observe no improvement with $K$ or $T$, indicating unimodality and unsuccessful posterior approximation. Consequently SELBO shifts uncertainty into the decoder, resulting in fuzzy, low confidence outputs (see \Cref{sec:experiments_appendix}). 
For comparison, in \Cref{sec:full_im_vae} we run the same experiment using full-image MNIST; results indicate that SIWAE provides benefits over SELBO in lower dimensional latent
spaces and these benefits diminish
as the dimensionality increases.


\vspace{-5pt}
\section{Conclusion}

We demonstrate that although stratified sampling enables ADVI with mixture posterior distributions, the ELBO impedes surrogate posterior multimodality. SIWAE, a tighter evidence lower bound analogous to the IWAE, utilizes stratification over posterior mixture components to make the bound tighter. We experimentally verify that SIWAE facilitates discovery of multimodality in the latent space, stratified ELBO does not, and that multimodality improves generative model performance, particularly for incomplete input data or low dimensionality representations.  We also show that SIWAE enables better classifier accuracy and calibration error and that both improve as as the number of components is increased.

\section*{Broader Impact}
Developing computationally efficient techniques which enable learning discrete structure in latent variable models is essential to the practitioner's ability to both understand and explain why a model makes the predictions it does. Moreover, offering the capacity to capture this structure--if truly present--implies a better calibrated model and therefore a more reliable model. Poorly understood models or inadequate calibration can result in misinformed decisions which may have dire consequences in certain situations.

\bibliographystyle{plainnat}
\small{
\bibliography{references}}

\begin{thebibliography}{33}
\providecommand{\natexlab}[1]{#1}
\providecommand{\url}[1]{\texttt{#1}}
\expandafter\ifx\csname urlstyle\endcsname\relax
  \providecommand{\doi}[1]{doi: #1}\else
  \providecommand{\doi}{doi: \begingroup \urlstyle{rm}\Url}\fi

\bibitem[Abadi et~al.(2016)Abadi, Barham, Chen, Chen, Davis, Dean, Devin,
  Ghemawat, Irving, Isard, et~al.]{Abadi:16}
Mart{\'\i}n Abadi, Paul Barham, Jianmin Chen, Zhifeng Chen, Andy Davis, Jeffrey
  Dean, Matthieu Devin, Sanjay Ghemawat, Geoffrey Irving, Michael Isard, et~al.
\newblock Tensorflow: A system for large-scale machine learning.
\newblock In \emph{12th $\{$USENIX$\}$ Symposium on Operating Systems Design
  and Implementation ($\{$OSDI$\}$ 16)}, pages 265--283, 2016.

\bibitem[Alemi et~al.(2017)Alemi, Fischer, Dillon, and Murphy]{Alemi:17}
Alex Alemi, Ian Fischer, Josh Dillon, and Kevin Murphy.
\newblock Deep variational information bottleneck.
\newblock In \emph{ICLR}, 2017.
\newblock URL \url{https://arxiv.org/abs/1612.00410}.

\bibitem[Bishop(1993)]{Bishop:94}
Christopher~M Bishop.
\newblock Novelty detection and neural network validation.
\newblock \emph{ICANN ’93}, 1993.

\bibitem[Bishop et~al.(1998)Bishop, Lawrence, Jaakkola, and Jordan]{Bishop:98}
Christopher~M Bishop, Neil~D Lawrence, Tommi Jaakkola, and Michael~I Jordan.
\newblock Approximating posterior distributions in belief networks using
  mixtures.
\newblock In \emph{Advances in neural information processing systems}, pages
  416--422, 1998.

\bibitem[Burda et~al.(2015)Burda, Roger, and Salakhutdinov]{Burda:16}
Yuri Burda, Roger~Grosse Roger, and Ruslan Salakhutdinov.
\newblock Importance weighted autoencoders.
\newblock \emph{arXiv preprint arXiv:1509.00519}, 2015.

\bibitem[Clevert et~al.(2015)Clevert, Unterthiner, and Hochreiter]{Clevert:15}
Djork-Arn{\'e} Clevert, Thomas Unterthiner, and Sepp Hochreiter.
\newblock Fast and accurate deep network learning by exponential linear units
  (elus).
\newblock \emph{arXiv preprint arXiv:1511.07289}, 2015.

\bibitem[Cremer et~al.(2017)Cremer, Morris, and Duvenaud]{Cremer:17}
Chris Cremer, Quaid Morris, and David Duvenaud.
\newblock Reinterpreting importance-weighted autoencoders.
\newblock \emph{arXiv preprint arXiv:1704.02916}, 2017.

\bibitem[Diederik et~al.(2014)Diederik, Welling, et~al.]{Kingma:14}
P~Kingma Diederik, Max Welling, et~al.
\newblock Auto-encoding variational bayes.
\newblock In \emph{Proceedings of the International Conference on Learning
  Representations (ICLR)}, volume~1, 2014.

\bibitem[Dilokthanakul et~al.(2016)Dilokthanakul, Mediano, Garnelo, Lee,
  Salimbeni, Arulkumaran, and Shanahan]{Dilokthanakul:16}
Nat Dilokthanakul, Pedro~AM Mediano, Marta Garnelo, Matthew~CH Lee, Hugh
  Salimbeni, Kai Arulkumaran, and Murray Shanahan.
\newblock Deep unsupervised clustering with {Gaussian} mixture variational
  autoencoders.
\newblock \emph{arXiv preprint arXiv:1611.02648}, 2016.

\bibitem[Doersch(2016)]{Doersch:16}
Carl Doersch.
\newblock Tutorial on variational autoencoders.
\newblock \emph{arXiv preprint arXiv:1606.05908}, 2016.

\bibitem[Domke and Sheldon(2019)]{Domke:19}
Justin Domke and Daniel~R Sheldon.
\newblock Divide and couple: Using {Monte Carlo} variational objectives for
  posterior approximation.
\newblock In \emph{Advances in Neural Information Processing Systems}, pages
  338--347, 2019.

\bibitem[Figurnov et~al.(2018)Figurnov, Mohamed, and Mnih]{Figurnov:18}
Mikhail Figurnov, Shakir Mohamed, and Andriy Mnih.
\newblock Implicit reparameterization gradients.
\newblock In \emph{Advances in Neural Information Processing Systems}, pages
  441--452, 2018.

\bibitem[Graves(2016)]{Graves:16}
Alex Graves.
\newblock Stochastic backpropagation through mixture density distributions.
\newblock \emph{arXiv preprint arXiv:1607.05690}, 2016.

\bibitem[Guo et~al.(2017)Guo, Pleiss, Sun, and Weinberger]{Guo:17}
Chuan Guo, Geoff Pleiss, Yu~Sun, and Kilian~Q Weinberger.
\newblock On calibration of modern neural networks.
\newblock In \emph{Proceedings of the 34th International Conference on Machine
  Learning-Volume 70}, pages 1321--1330. JMLR. org, 2017.

\bibitem[Higgins et~al.(2017)Higgins, Matthey, Pal, Burgess, Glorot, Botvinick,
  Mohamed, and Lerchner]{Higgins:17}
Irina Higgins, Loic Matthey, Arka Pal, Christopher Burgess, Xavier Glorot,
  Matthew Botvinick, Shakir Mohamed, and Alexander Lerchner.
\newblock {Beta-VAE}: Learning basic visual concepts with a constrained
  variational framework.
\newblock \emph{ICLR}, 2\penalty0 (5):\penalty0 6, 2017.

\bibitem[Jankowiak and Obermeyer(2018)]{Jankowiak:18}
Martin Jankowiak and Fritz Obermeyer.
\newblock Pathwise derivatives beyond the reparameterization trick.
\newblock \emph{arXiv preprint arXiv:1806.01851}, 2018.

\bibitem[Jiang et~al.(2017)Jiang, Zheng, Tan, Tang, and Zhou]{Jiang:17}
Zhuxi Jiang, Yin Zheng, Huachun Tan, Bangsheng Tang, and Hanning Zhou.
\newblock Variational deep embedding: An unsupervised and generative approach
  to clustering.
\newblock In \emph{Proceedings of the 26th International Joint Conference on
  Artificial Intelligence}, IJCAI’17, page 1965–1972. AAAI Press, 2017.
\newblock ISBN 9780999241103.

\bibitem[Johnson et~al.(2016)Johnson, Duvenaud, Wiltschko, Adams, and
  Datta]{Johnson:16}
Matthew~J Johnson, David~K Duvenaud, Alex Wiltschko, Ryan~P Adams, and
  Sandeep~R Datta.
\newblock Composing graphical models with neural networks for structured
  representations and fast inference.
\newblock In D.~D. Lee, M.~Sugiyama, U.~V. Luxburg, I.~Guyon, and R.~Garnett,
  editors, \emph{Advances in Neural Information Processing Systems 29}, pages
  2946--2954. Curran Associates, Inc., 2016.
\newblock URL
  \url{http://papers.nips.cc/paper/6379-composing-graphical-models-with-neural-networks-for-structured-representations-and-fast-inference.pdf}.

\bibitem[Jordan et~al.(1999)Jordan, Ghahramani, Jaakkola, and Saul]{Jordan:99}
Michael~I Jordan, Zoubin Ghahramani, Tommi~S Jaakkola, and Lawrence~K Saul.
\newblock An introduction to variational methods for graphical models.
\newblock \emph{Machine learning}, 37\penalty0 (2):\penalty0 183--233, 1999.

\bibitem[Kingma and Ba(2014)]{Adam}
Diederik~P Kingma and Jimmy Ba.
\newblock Adam: A method for stochastic optimization.
\newblock \emph{arXiv preprint arXiv:1412.6980}, 2014.

\bibitem[LeCun et~al.(1998)LeCun, Bottou, Bengio, and Haffner]{Lecun:98}
Yann LeCun, L{\'e}on Bottou, Yoshua Bengio, and Patrick Haffner.
\newblock Gradient-based learning applied to document recognition.
\newblock \emph{Proceedings of the IEEE}, 86\penalty0 (11):\penalty0
  2278--2324, 1998.

\bibitem[Nalisnick et~al.(2018)Nalisnick, Matsukawa, Teh, Gorur, and
  Lakshminarayanan]{Nalisnick:18}
Eric Nalisnick, Akihiro Matsukawa, Yee~Whye Teh, Dilan Gorur, and Balaji
  Lakshminarayanan.
\newblock Do deep generative models know what they don't know?
\newblock \emph{arXiv preprint arXiv:1810.09136}, 2018.

\bibitem[Oh et~al.(2019)Oh, Gallagher, Murphy, Schroff, Pan, and
  Roth]{oh2018modeling}
Seong~Joon Oh, Andrew~C. Gallagher, Kevin~P. Murphy, Florian Schroff, Jiyan
  Pan, and Joseph Roth.
\newblock Modeling uncertainty with hedged instance embeddings.
\newblock In \emph{International Conference on Learning Representations}, 2019.
\newblock URL \url{https://openreview.net/forum?id=r1xQQhAqKX}.

\bibitem[Poduval et~al.(2020)Poduval, Loya, Patel, and Jain]{Poduval:20}
Pranav Poduval, Hrushikesh Loya, Rajat Patel, and Sumit Jain.
\newblock Mixture distributions for scalable {Bayesian} inference, 2020.
\newblock URL \url{https://openreview.net/forum?id=S1x6TlBtwB}.

\bibitem[Rainforth et~al.(2018)Rainforth, Kosiorek, Le, Maddison, Igl, Wood,
  and Teh]{Rainforth:18}
Tom Rainforth, Adam~R Kosiorek, Tuan~Anh Le, Chris~J Maddison, Maximilian Igl,
  Frank Wood, and Yee~Whye Teh.
\newblock Tighter variational bounds are not necessarily better.
\newblock \emph{arXiv preprint arXiv:1802.04537}, 2018.

\bibitem[Rezende et~al.(2014)Rezende, Mohamed, and Wierstra]{Rezende:14}
Danilo~Jimenez Rezende, Shakir Mohamed, and Daan Wierstra.
\newblock Stochastic backpropagation and approximate inference in deep
  generative models.
\newblock In \emph{International Conference on Machine Learning}, pages
  1278--1286, 2014.

\bibitem[Roeder et~al.(2017)Roeder, Wu, and Duvenaud]{Roeder:17}
Geoffrey Roeder, Yuhuai Wu, and David~K Duvenaud.
\newblock Sticking the landing: Simple, lower-variance gradient estimators for
  variational inference.
\newblock In \emph{Advances in Neural Information Processing Systems}, pages
  6925--6934, 2017.

\bibitem[Ruiz et~al.(2016)Ruiz, Titsias, and Blei]{Ruiz:16}
Francisco J.~R. Ruiz, Michalis~K. Titsias, and David~M. Blei.
\newblock The generalized reparameterization gradient.
\newblock In \emph{Advances in neural information processing systems}, pages
  460--468, 2016.

\bibitem[Salimans and Knowles(2013)]{Salimans:13}
Tim Salimans and David~A Knowles.
\newblock Fixed-form variational posterior approximation through stochastic
  linear regression.
\newblock \emph{Bayesian Analysis}, 8\penalty0 (4):\penalty0 837--882, 2013.

\bibitem[Snoek et~al.(2019)Snoek, Ovadia, Fertig, Lakshminarayanan, Nowozin,
  Sculley, Dillon, Ren, and Nado]{Snoek:19}
Jasper Snoek, Yaniv Ovadia, Emily Fertig, Balaji Lakshminarayanan, Sebastian
  Nowozin, D~Sculley, Joshua Dillon, Jie Ren, and Zachary Nado.
\newblock Can you trust your model's uncertainty? evaluating predictive
  uncertainty under dataset shift.
\newblock In \emph{Advances in Neural Information Processing Systems}, pages
  13969--13980, 2019.

\bibitem[Tomczak and Welling(2017)]{Tomczak:17}
Jakub~M Tomczak and Max Welling.
\newblock {VAE} with a {VampPrior}.
\newblock \emph{arXiv preprint arXiv:1705.07120}, 2017.

\bibitem[Tucker et~al.(2019)Tucker, Lawson, Gu, and Maddison]{Tucker:19}
George Tucker, Dieterich Lawson, Shixiang Gu, and Christopher Maddison.
\newblock Doubly reparameterized gradient estimators for monte carlo
  objectives.
\newblock 2019.
\newblock URL \url{https://openreview.net/pdf?id=HkG3e205K7}.

\bibitem[West(1993)]{West:93}
Mike West.
\newblock Approximating posterior distributions by mixtures.
\newblock \emph{Journal of the Royal Statistical Society: Series B
  (Methodological)}, 55\penalty0 (2):\penalty0 409--422, 1993.

\end{thebibliography}

\clearpage
\newpage
\appendix



\newpage

\section{Theorems and proofs}\label{sec:theorems}

\begin{thm}
$\mathcal{L}_{\text{SIWAE}}^T$ is a lower bound on the evidence $\log p(x)$.
\begin{proof}
\begin{align*}
&\log p(x)=\\
&= \log \frac{1}{T} \sum_{t=1}^T \E_{z_t \sim q_\phi(z|x)} \left[ \frac{p(x|z_t)r(z_t)}{q_\phi(z_t|x)} \right]\\
&= \log \frac{1}{T} \sum_{t=1}^T \sum_{k=1}^K \alpha_{k, \phi}(x) \E_{z_{kt} \sim q_{k, \phi}(z |x)} \left[ \frac{p(x|z_{kt})r(z_{kt})}{q_\phi(z_{kt}|x)} \right]\\
&= \log \int \cdots \int \left[\prod_{t=1}^T \prod_{k=1}^K dz_{kt} q_{k, \phi}(z_{kt}|x) \right] \times \\
& ~{ } \quad\qquad \hphantom{ \log } \frac{1}{T} \sum_{t=1}^T \sum_{k=1}^K \alpha_{k, \phi}(x) \frac{ r(z_{kt}) p(x|z_{kt})}{q_\phi(z_{kt}|x)}\\
&\ge \int \cdots \int \left[\prod_{t=1}^T \prod_{k=1}^K dz_{kt} q_{k, \phi}(z_{kt}|x) \right] \times \\
& ~{ } \quad\qquad  \log  \frac{1}{T} \sum_{t=1}^T \sum_{k=1}^K \alpha_{k, \phi}(x)\frac{p(x | z_{kt})r(z_{kt})}{q_\phi(z_{kt}|x)}\\
&= \E_{\{z_{kt} \sim q_{k, \phi}(z | x)\}_{k=1,t=1}^{K, T}} \Bigg[\\
& \log \frac{1}{T} \sum_{t = 1}^T \sum_{k = 1}^K \alpha_{k, \phi}(x) \frac{p(x|z_{kt})r(z_{kt})}{q_{\phi}(z_{kt} | x)} \Bigg] \\
&\equiv \mathcal{L}_{\text{SIWAE}}^T(\phi)
\end{align*}

\end{proof}
\label{thm:theorem1}
\end{thm}

\begin{thm}
When $K > 1$, $\mathcal{L}_\text{SIWAE}^T$ is a tighter lower bound than $\mathcal{L}_\text{IWAE}^T$.
\end{thm}

\begin{proof}
\begin{align*}
&\mathcal{L}_\text{SIWAE}^T(\phi) \equiv\nonumber\\
&\equiv \E_{\{z_{kt} \sim  q_{k, \phi}(z_k|x)\}_{t=1,k=1}^{T,K}} \left[ \vphantom{\sum_{k=1}^K} \right.\\
&~{ } \quad\qquad \left. \log \frac{1}{T} \sum_{t=1}^T \sum_{k=1}^K  \alpha_{k, \phi}(x) \frac{p(x|z_{kt}) r(z_{kt})}{q_\phi(z_{kt}|x)}\right]\\
&\ge \sum_{k=1}^K  \alpha_{k, \phi}(x) \E_{\{z_{kt} \sim  q_{k, \phi}(z_{kt}|x)\}_{t=1,k=1}^{T,K}} \left[ \vphantom{\sum_{k=1}^K} \right.\\
&~{ } \quad\qquad \left. \log \frac{1}{T} \sum_{t=1}^T  \frac{ p(x|z_{kt})r(z_{kt})}{q_\phi(z_{kt}|x)}\right]\\
&= \E_{\{z_t \sim  q_\phi(z|x)\}_{t=1}^{T}} \left[ \log \frac{1}{T} \sum_{t=1}^T  \frac{ p(x | z_t)r(z_t) }{q_\phi(z_t|x)}\right] \\
&\equiv \mathcal{L}_\text{IWAE}^T(\phi)
\end{align*}
\label{thm:theorem2}
\end{proof}

\section{Experimental Details}\label{sec:experiments_details}
\subsection{Single Column MNIST Classification}
\textbf{Architecture.}  For our experiments, we use an MLP architecture with 4 layers of 128 hidden units and ELU activation functions \citep{Clevert:15} for the encoder. The last layer predicts the parameters for a distribution over a two dimensional latent variable (in \autoref{sec:mnist-1d}, we run the same experiment but with a one dimensional latent variable). For all models, we use a mixture of $K$ multivariate normal distributions with full covariance with mixture weights as a learnable parameter which is predicted by the encoder.  For models with $K=1$, this reduces to a single multivariate normal distribution with no learnable mixture weights.  For the decoder, we use an affine transformation that outputs the logits for a categorical distribution. We use such a simple architecture for decoder to encourage the encoder to capture potentially multimodal information about the class of an image. For our prior distribution $r(z)$, we use a trainable mixture of Gaussians, although we found the prior makes relatively little difference in the final results.  

\textbf{Training Procedure. } For a single component model, we optimize both the traditional evidence lower bound (ELBO), as well as the importance weighted estimate of the evidence (IWAE). For the mixture models, we use stratified sampling to compute the ELBO (SELBO), as well as the Stratified-IWAE (SIWAE) derived in Section~\ref{sec:approach}.  We use $K=[1, 2, 5, 10]$ for the number of mixture components, and $T=[1, 2, 5, 10]$ for the number of samples drawn \emph{per component}. To regulate the information content of the posterior, we use a $\beta=0.05$ penalty on the KL divergence term (and the equivalent term in the SIWAE objective), as used in \cite{Higgins:17}.  Because one-column MNIST does not have an established benchmark, we also train two deterministic models to use as baselines: (1) a ``pyramid'' MLP with 5 layers of 256 hidden units to approximate the peak deterministic accuracy, and (2) a ``bottleneck'' MLP with the same architecture as our VIB models, therefore containing a two dimensional ``latent space.''  All models were trained for 50 epochs using the Adam optimizer \cite{Adam} with a learning rate of 0.001 which was decayed by 0.5 every 15000 train steps. When training SELBO models, $T$ refers to the number of samples drawn to compute the Monte-Carlo estimate of the objective.

\textbf{Evaluation. } To evaluate the accuracy of the model, we first need the posterior predictive distribution. We sample the posterior predictive by decoding $10^{4}$ samples from $q_{\phi}(z|x)$ and averaging the class probabilities returned by each sample. This marginalizes over the uncertainty in the latent variables and if our prior beliefs are correct, nominally produces calibrated probabilities. From these probabilities, we take the highest-probability class, and consider that the prediction of the model. Accuracy is then defined as the number of correct predictions divided by the total number of examples in the test set.

We also compute the \emph{Expected Calibration Error} \citep[ECE; ][]{Guo:17}.  For this, we decode 1000 samples from the posterior and compute the average probability of each class.  We take the model prediction to then be \begin{equation}
    \hat{y}=\text{argmax}p(y|z)
\end{equation}
This prediction is labeled correct if it is equal to the true class label $y$, otherwise it is labeled incorrect.  In addition to checking if each prediction is correct, we also get the predicted confidence for the true class $p(y_{true})$.  We then rank our data and divide into 10 bins such that each bin contains 10\% of the examples, ranked by confidence in the true class $p(y_{true})$.  The confidence of a bin is computed as $p^{bin}(y_{true}=\frac{1}{n}\sum_{i}^{n}p_(y_{true}^{(n)})$.  The probability of the truth for a given bin is given by the fraction of predictions in the bin which were correct.  The expected calibration error is then defined as the average absolute value of difference between the confidence in a bin and the probability of correctness in that bin.

\section{Additional Experiment Results}

\subsection{Toy Problem}

In \autoref{fig:toy-dataset} we visualize the data from the toy experiment and the training curves for the ELBO estimators.

\begin{figure*}[t]
    \centering
    \begin{subfigure}[t]{0.5\textwidth}
    \includegraphics[width=\hsize]{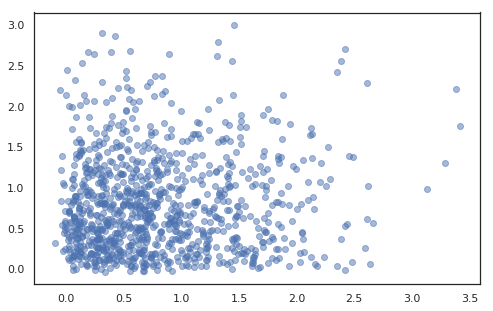}
    \end{subfigure}%
    ~
    \begin{subfigure}[t]{0.5\textwidth}
    \includegraphics[width=\hsize]{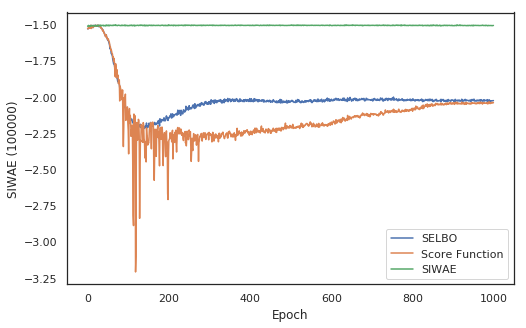}
    \end{subfigure}
    \caption{On the left is a toy dataset generated by sampling $z \sim \mathcal{N}(0, I); x \sim \mathcal{N}(|z|, \sigma^2I)$. On the right are SIWAE values at each epoch while training posteriors using SELBO, SIWAE, and score function estimators of the evidence. Due to mixture components collapse, the SELBO and score function posteriors achieve lower values of SIWAE.}
    \label{fig:toy-dataset}
\end{figure*}

\begin{figure*}[h]
\centering
    \includegraphics[width=0.5\textwidth]{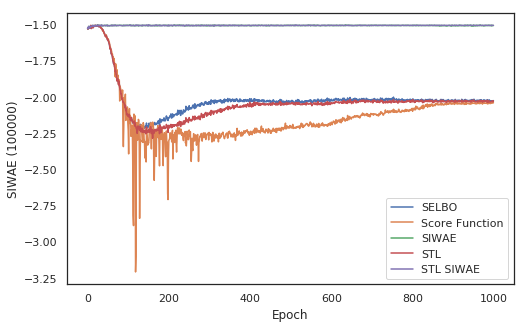}
    \caption{Training results for the toy experiment with the addition of ``sticking-the-landing''
    versions of SELBO and SIWAE. We observe no significant difference between the final training 
    SIWAEs of STL-SELBO vs. SELBO (-2.026 vs. -2.024 respectively) and STL-SIWAE vs. SIWAE (-1.505 and -1.505 respectively)}
    \label{fig:toy-dataset-stl}
\end{figure*}

\begin{figure*}[h]
    \centering
    \begin{subfigure}[t]{0.5\textwidth}
    \includegraphics[width=\hsize]{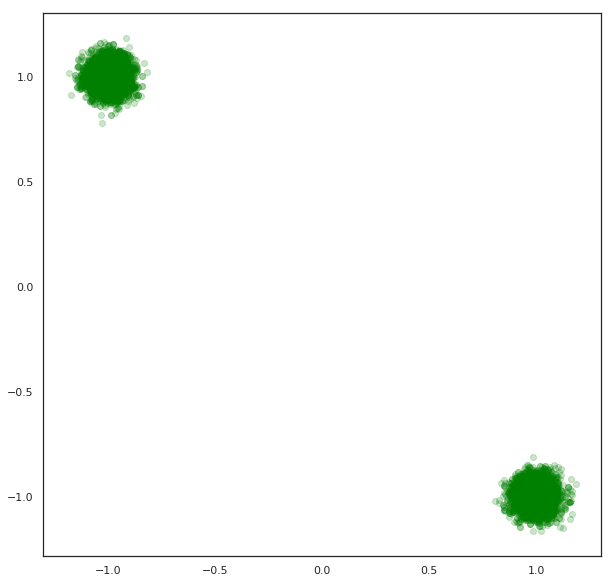}
    \caption{STL-SELBO}
    \end{subfigure}%
    ~
    \begin{subfigure}[t]{0.5\textwidth}
    \includegraphics[width=\hsize]{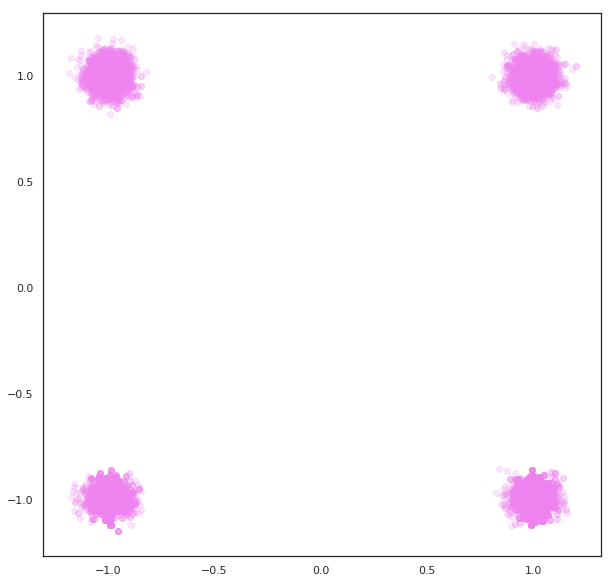}
    \caption{STL-SIWAE posterior}
    \end{subfigure}%
    \caption{Samples from of the learned (implicit) posteriors for the observed data point $(1, 1)$ for the STL-SELBO and STL-SIWAE \citep{Roeder:17}.}
    \label{fig:toy-posteriors-stl}
\end{figure*}

In addition to the experiments presented in the main body of the paper, we also ran a comparison to the "sticking the landing" (STL), pathwise derivative estimator, which results in reduced variance in the model gradients (with a potential increase in bias \cite{Tucker:19}).  Our main interest lies in determining if the STL gradient estimator is itself sufficient for fitting multimodal posteriors, or if the use of SIWAE is truly necessary for inferring multimodality.  We ran our test on the toy problem using STL to evaluate both the SELBO and the SIWAE losses.  We show the evidence, as measured by a $10^5$ sample SIWAE as a function of training epochs in \Autoref{fig:toy-dataset-stl}.  We find that for both SELBO and SIWAE, the model evidence is unchanged by using the STL gradient estimator, indicating that STL does not help in converging to a better model.  Furthermore, in \Autoref{fig:toy-posteriors-stl}, we show samples from the learned posterior.  We find that using SELBO, even with STL, results in a model which does not discover all modes in the posterior.  The fact that SELBO and SIWAE give the same results as STL-SELBO and STL-SIWAE suggests that it is the SIWAE loss itself, rather than the gradient estimator, that is providing the necessary ingredients for detecting multimodality.  However, we speculate that STL may offer more relative improvement in situations where the bias introduced by SELBO is low compared to the variance introduced by SIWAE.

\subsection{Single Column MNIST Classification}
\begin{figure*}[ht]
    \centering
    \includegraphics[width=\hsize]{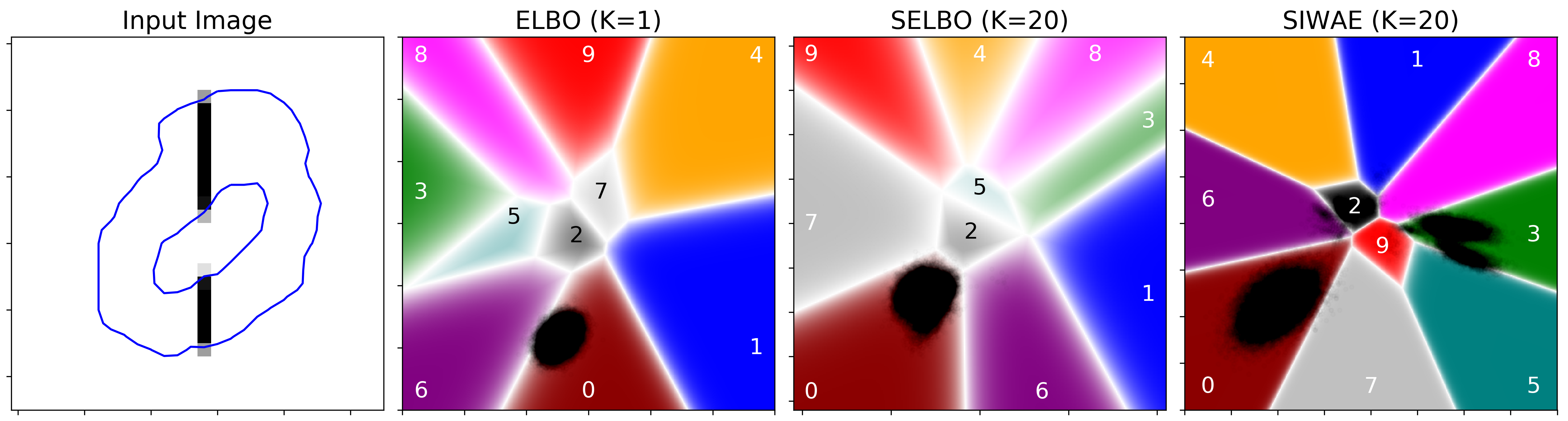}
    \includegraphics[width=\hsize]{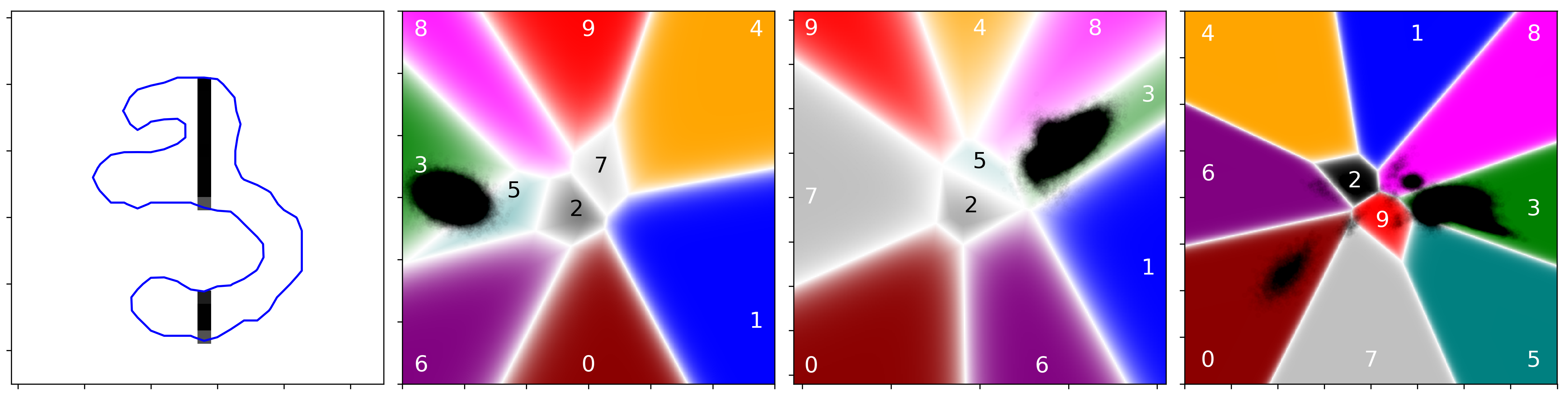}
    \caption{Similar to \autoref{fig:local_elbo_latent}, but for two examples where the input to the model is extremely similar even though the inputs are from different classes.  We find that while all models infer the correct class, models trained with SIWAE are better suited to recognize the similarity between these two images, assigning some probability to the other class.}
    \label{fig:latent_space_03}
\end{figure*}
In the main text of the paper, we showed the latent space distribution for an image wherein the ambiguity introduced by the use of a single column in the inference resulted in a multimodal latent space.  Furthermore we showed that SIWAE was able to detect and capture this multimodality much better than ELBO or SELBO, which either are structurally unequipped to do so (ELBO), or which are penalized for doing so (SELBO).  To show that the capacity for multimodality aids in the interpretability of our model, consider the images shown in \Autoref{fig:latent_space_03}.  Both images, while having quite different true appearances, appear nearly identical when viewed as only their center column.  Therefore, a model should classify this pair as "either a 0 or a 3", since both of these classes have this appearance.  However, this is not observed when SELBO is used.  The model (correctly) predicts a zero for the top image, and a three for the bottom image, with no indication that the other is a possibility.  In contrast, the SIWAE model also predicts the correct class, but correctly assigns a non-negligible fraction of its samples to the other class.  In this sense, uncertainty is measured in the latent space itself using the posterior distribution.

\subsection{One-dimensional latent variable}\label{sec:mnist-1d}

\begin{figure*}[t]
    \centering 
    \begin{subfigure}[t]{0.4\textwidth}
    \centering 
    \includegraphics[height=7cm]{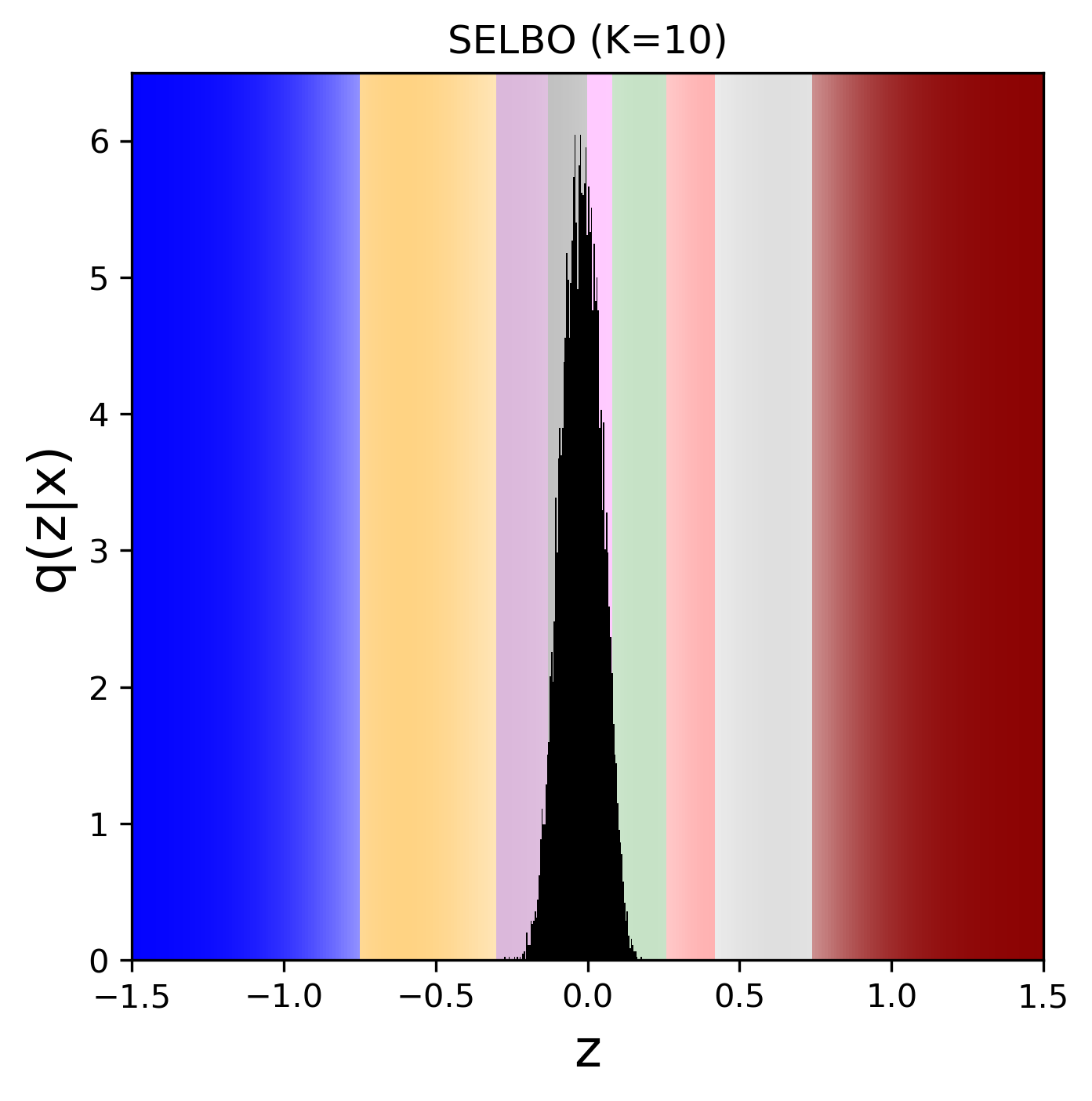}
    \end{subfigure}%
    ~
    \begin{subfigure}[t]{0.4\textwidth}
    \centering 
    \includegraphics[height=7cm]{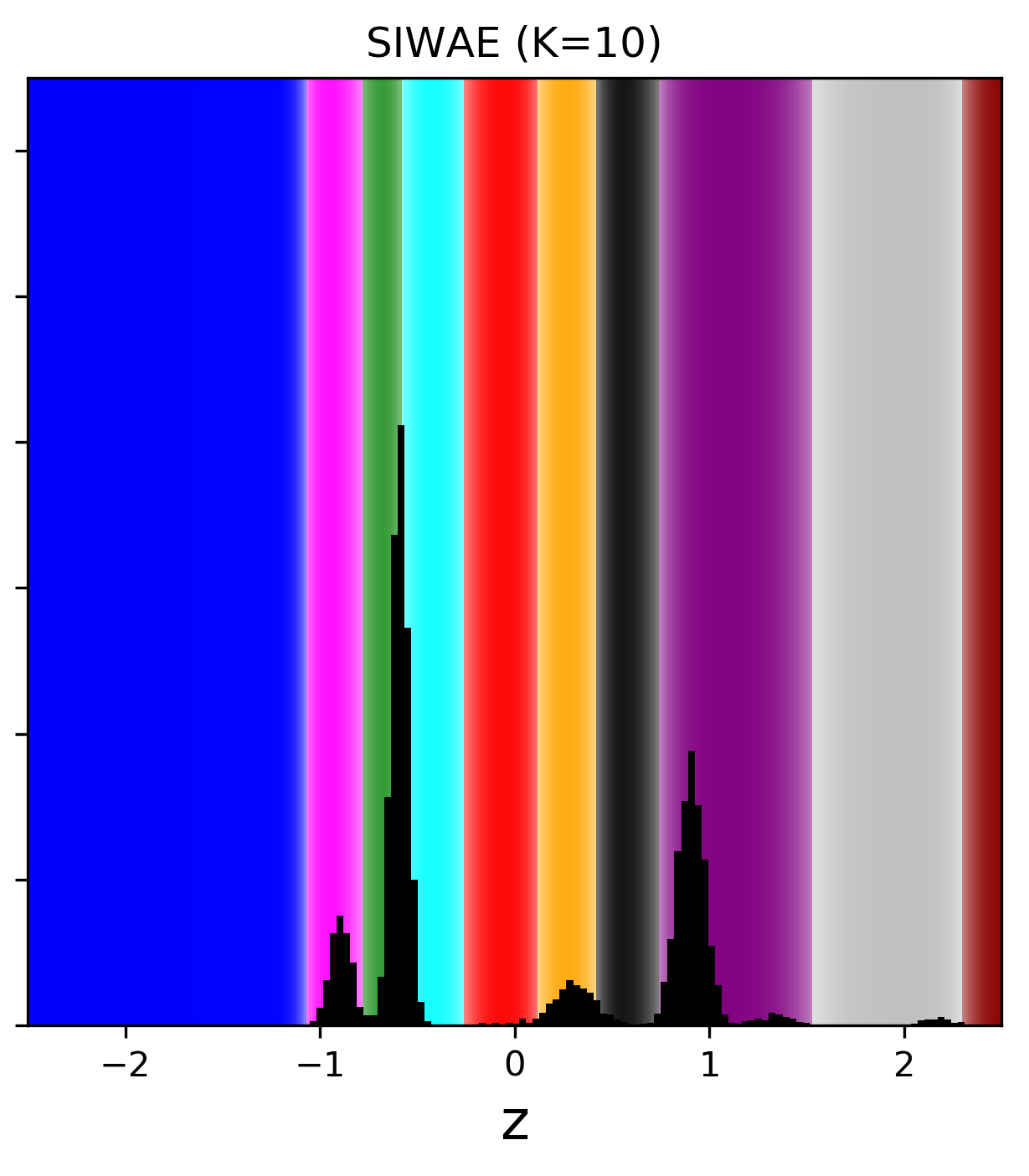}
    \end{subfigure}%
    \caption{Similar to \autoref{fig:local_elbo_latent}, but using a one-dimensional latent variable.  The input example is the same as in \autoref{fig:local_elbo_latent}.  The left panel shows the latent representation found by optimizing the SELBO objective.  Only a single mode is identified in the latent space.  In contrast, optimizing the SIWAE objective produces a latent representation with multiple distinct modes.}
    \label{fig:one_dim}
\end{figure*}

We ran the same MNIST classification experiment using a one-dimensional latent variable. In general, fitting a one-dimensional latent variable should result in an appreciable drop in accuracy because a single-dimensional bottleneck allows for a maximum of two decision boundaries for a particular class, and therefore forces a latent representation which becomes multimodal in the presence of any complex structure in the uncertainty. Therefore, a reasonable expectation is that this should force a model trained with SELBO to learn distinct modes in the encoder.  However, we experimented with training for this objective using multiple components as well as with a single component, and in all cases achieve an accuracy of 51\% or lower, which is substantially worse than can be reached in two dimensions. By dissecting this model, we see again that the model reduces to a single mode in the posterior, either by assigning all of the component weights to a single mode, or by merging all of the separate modes together (\autoref{fig:one_dim}). This strongly suggests that the SELBO objective actively opposes the formation of multiple modes in the posterior.

Using the SIWAE objective instead of the SELBO, we see our accuracy climb to $76\%$, nearly equivalent to the peak accuracy in two dimensions. We also find that using the SIWAE objective with only a single component (.e. IWAE) outperforms the traditional ELBO substantially as well, achieving 63\% accuracy. However, there is still a substantial gap between the IWAE model and the SIWAE model. All of the probabilistic models outperform a deterministic model in this case, which achieves a peak accuracy of merely 25\%.

\subsection{Single Column MNIST VAE}\label{sec:experiments_appendix}

\begin{figure}[htb]
    \centering
    \includegraphics[width=\hsize]{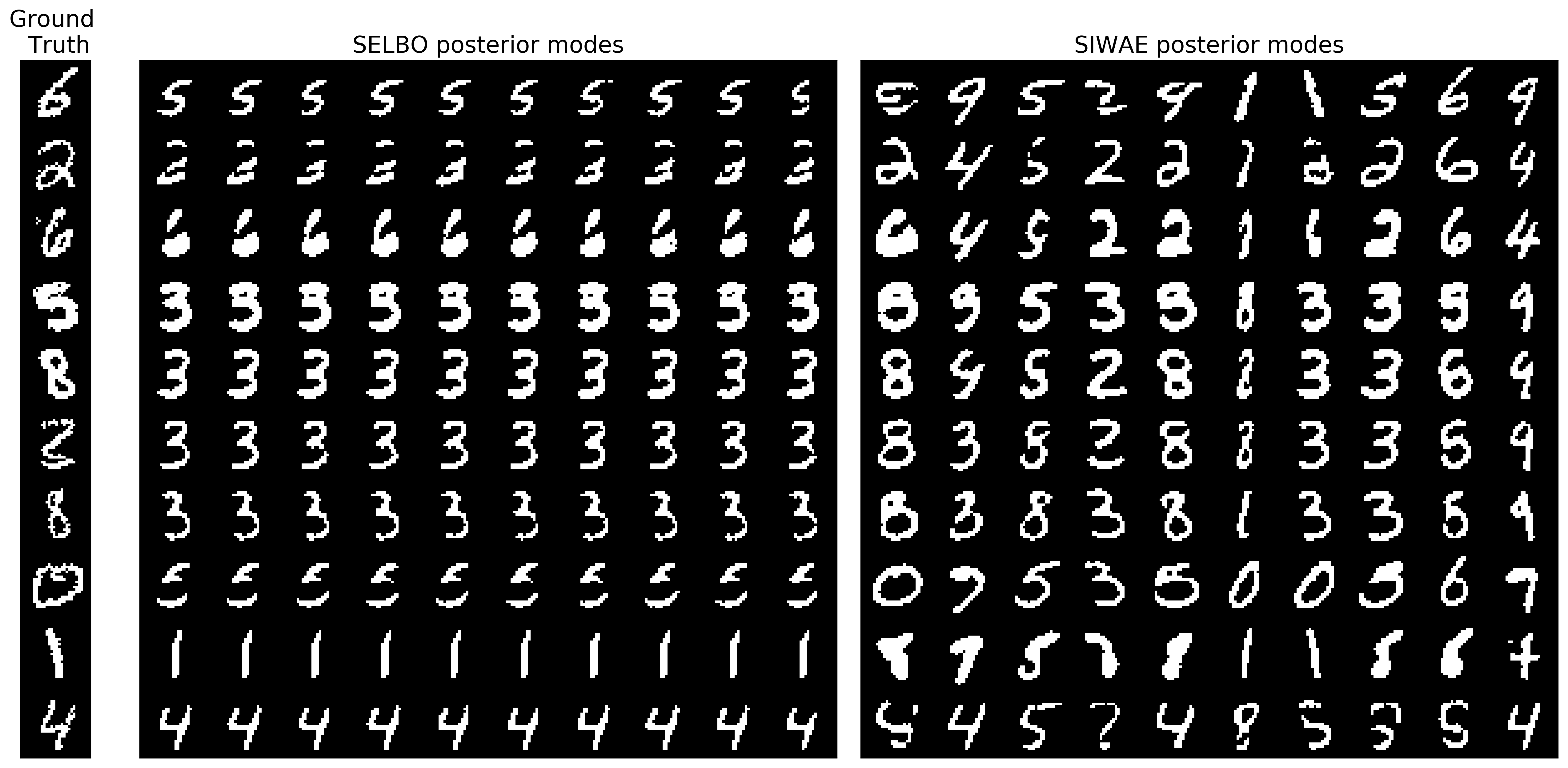}
    \caption{A visualization of modes from the posterior distribution of models trained with SELBO, compared to models trained with SIWAE.  The leftmost column shows the true image, from which the centermost column was fed to the encoder.  The central block of images shows the reconstructed modes of the posterior for a model trained with SELBO.  The rightmost block of images shows reconstructed modes from the posterior of a model trained with SIWAE.  Each column corresponds to a different mode in the posterior.  The SELBO modes all appear identical, suggesting that the model is not leveraging multimodality.  In contrast, the SIWAE models learn a diverse assortment of modes, which offer competing explanations of the output data.  Note that at least one of the SIWAE modes provides an accurate description of the data, while the same cannot be said for SELBO modes.}
    \label{fig:single_column_vae_ims}
\end{figure}
One qualitative metric that helps to asses the performance of VAEs, independent of quantitative metrics such as the log-evidence, which may not always be entirely informative, is the appearance of images drawn from the model.  This can be examined in multiple ways.  One can draw samples from the generative model to see if they appear qualitatively similar to images from the dataset.  Alternatively, one can generate reconstructions using samples from the inference model to see if the inference appears reasonable.  To assess the performance of SIWAE compared to SELBO, we found it more reasonable to examine the inference model.

To generate templates of each mode from the posterior, we passed the mode of the encoder component distributions to the decoder, and took the mode of each decoded image.  These are shown in \Autoref{fig:single_column_vae_ims}.  We find that all modes from SELBO make roughly the same prediction, showing that the modes have collapsed together.  We also find that these modes often do not capture the correct appearance of the input data over any component of the encoder.  In some cases, this may be an example of the decoder trying to "hedge its bets" to make up for the inability of SELBO to recover multimodality, and therefore predicting nearly 0.5 for pixels which have competing explanations.  However, it should be pointed out that this cripples the generative model, as the samples produced are of lower quality than they could otherwise be if uncertainty were represented correctly.  

In contrast, SIWAE does not encounter issues with collapsing modes, and produces multiple different explanations for each instance fed to the encoder.  Notably, we find that for more unambiguous inputs (e.g. zeros), the encoder produces multiple template images from the correct class, but having different stylistic appearance.  For images which may be explained by multiple different classes, we find that the modes produce a census of the potential classes.  We find that at least one mode will typically provide a good explanation for the output data, with some exceptions occurring for rarer images (such as a crossed seven, which only occurs in roughly 10\% of sevens).  For the same reason, we suspect that samples drawn from the generative model will exhibit qualitative appearance more indicative of examples from the dataset.  To this end, the ability to represent multimodality has prevented the generative model from being hindered by the inference model, as is observed in models using SELBO.

\subsection{Full Image VAE}
\label{sec:full_im_vae}
Our previous experiments all indicate that SIWAE offers advantages over SELBO when the input data does not contain sufficient information to unambiguously determine the output quantity.  However, when the input data is not ambiguous, these advantages are no longer present and SIWAE may therefore offer fewer relative improvements compared to SELBO.  At the same time, experiments comparing IWAE to ELBO indicate that losses like SIWAE may also exhibit higher variance in the gradients, which may affect the training dynamics of the model, resulting in worse outcomes \cite{Rainforth:18}.  While our previous experiments have shown that higher variance in the gradients is overcome by the ability to successfully leverage multimodality when it exists, it is logical that higher variance gradients could become a detriment to the relative performance when multimodality offers no advantages.  We therefore expect that SELBO should perform equal to or better than SIWAE in clean and simple problems where there is no need for multimodality in the latent representation of the data.

To that end, we trained a VAE for 100 epochs on the standard binarized MNIST benchmark dataset, with no corruption of the inputs.  Similar to our previous experiment, we evaluate performance as a function of $K$ and $T$, and the dimensionality of the latent space using a 100 sample SIWAE estimate of the evidence.  Details of our training procedure and experimental results can be found in the supplementary material (\autoref{sec:training_supplement} and \autoref{sec:experiments_appendix}).  

We find that when the latent space is low dimensionality, the results are similar to our results from previous experiments.  Model performance is improved by using SIWAE instead of SELBO, with the performance improving by increasing either $K$ or $T$ (though the improvement with $T$ is ambiguous).  This makes sense, as the encoder is able to overcome the limitations imposed by low dimensionality by using multimodality to represent complex nonlinear structure.  In higher dimensionality however,  we find that SELBO performs better than SIWAE.  This also makes sense intuitively:  as the number of dimensions increases, so too does the number of ways in which two unimodal entities can differ.  Therefore, the advantages that multimodality provides in low dimensions no longer exist as the dimensionality gets sufficiently large.  We therefore expect that using SELBO in larger encoding spaces gives qualitatively better results, though this comes at the cost of explainability. In this regard we present SIWAE and SELBO as two different tools enabling exploration of two different regimes.

\begin{figure}[htb]
    \centering
    \includegraphics[width=\hsize]{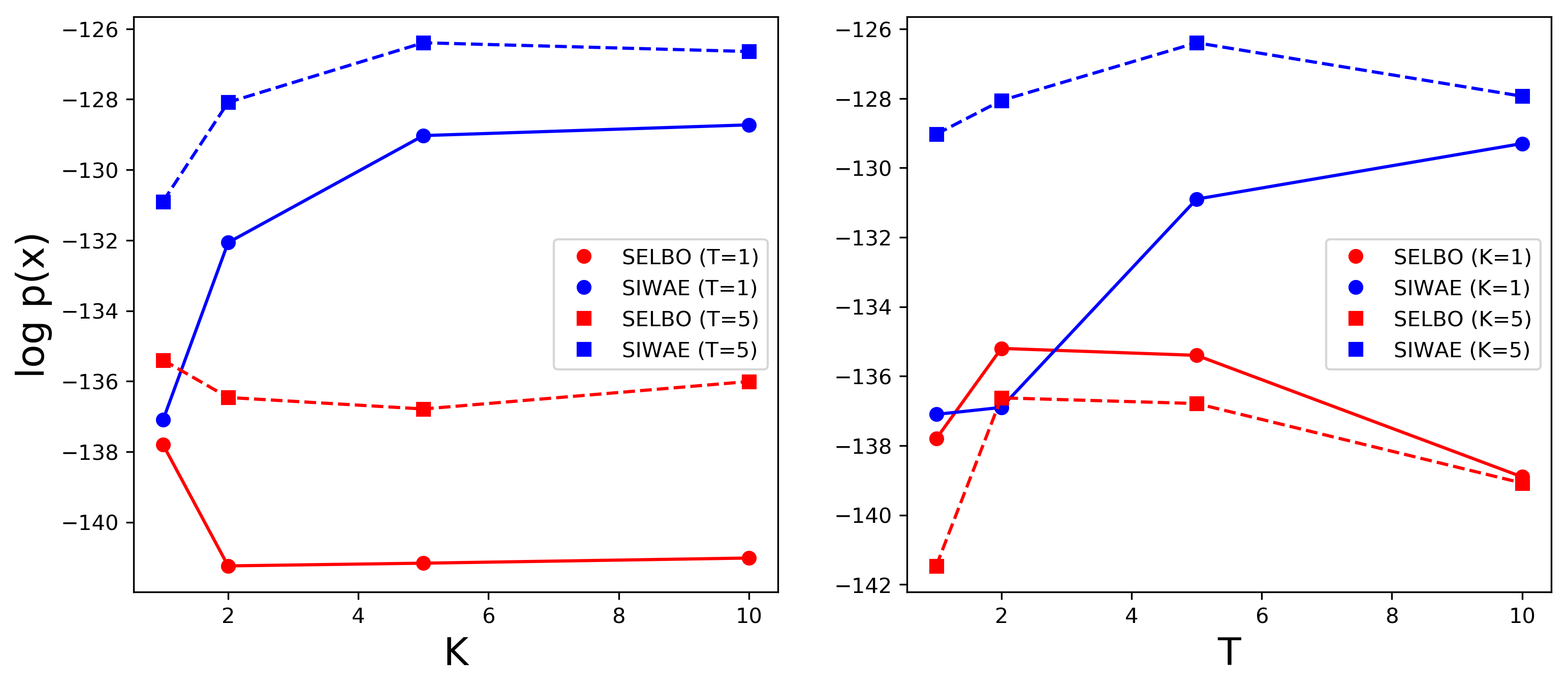}
    \caption{Same as Figure~\ref{fig:singlecolumnevidence}, but for a model given the full MNIST image as an input, and with a 2 dimensional latent space.  We find that SIWAE models improve with increasing either $K$ or $T$, while SELBO models appear to have ambiguous improvement with either $K$ or $T$.}
    \label{fig:fullimVAE2}
\end{figure}

\subsection{MNIST Style Modality}\label{sec:mnist-style}

\begin{figure*}[t]
\centering
\includegraphics[width=\hsize]{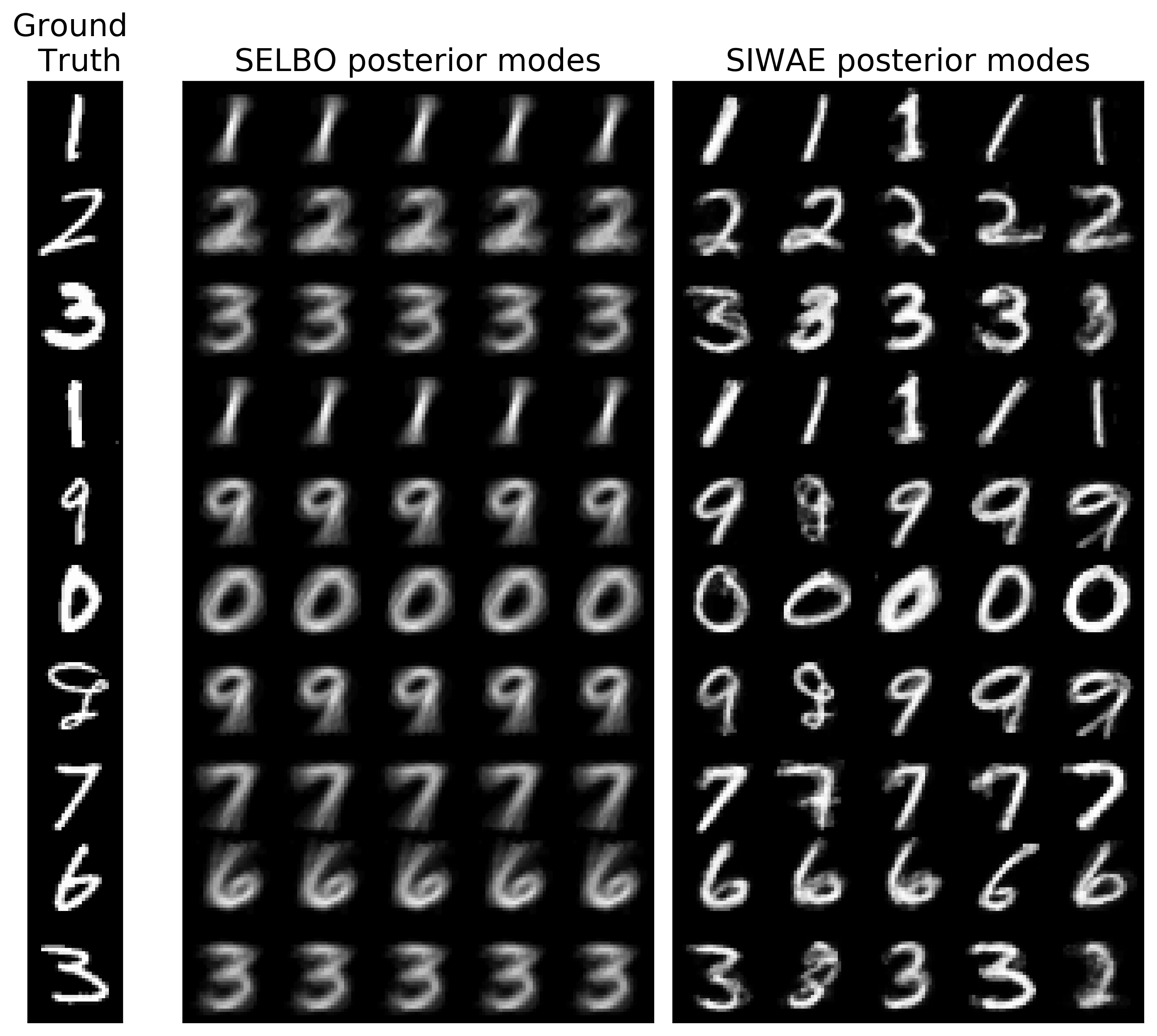}
\caption{Posterior modes learned by SELBO and SIWAE when trained to predict a random image from the same class.  SELBO learns only to infer the class of the image, producing a fuzzy reconstruction that is able to explain all different styles from that class simultaneously.  SIWAE instead learns to encode multiple different styles for each image.  This results in perceptually sharper reconstructions, and also in a better capture of uncertainty in the data.}
\label{fig:style_modes}
\end{figure*}

Thus far, we have shown that when input information to the encoder is limited, SELBO is unable to offer competing explanations for data.  We have shown in \autoref{sec:VIB} that this is a detriment to model calibration and we have also shown in \autoref{sec:single_column_vae} that this limits generative model performance.  In these two experiments, the model effectively had to represent class-specific explanations as each mode in the posterior.  However, equally interesting is if multimodality can be used to represent style in images.

To test this, we trained a VAE on MNIST, where the output target was a randomly chosen image from the same class as the input.  This effectively gives the model unambiguous information as to the class, but is completely ambiguous with respect to style.  Our hypothesis is that because SIWAE can provide multiple explanations for an input, that it will produce multiple images with different styles.  SELBO meanwhile would be penalized for producing multiple explanations, and would therefore produce a single fuzzy image for the output.

In \autoref{fig:style_modes}, we show the decoder means of the encoder style modes learned by SELBO and SIWAE models.  As expected, we find that SELBO learns only a single style mode, with all 5 possible encoder components producing roughly the same image, indicating that they have collapsed together.  Furthermore, the single mode learned by SELBO appears fuzzy, indicating that uncertainty in the output pixels is being explained by the decoder.  This makes intuitive sense:  SELBO penalizes any posterior mode for providing an explanation that is incorrect, even if that explanation is reasonable.  The model therefore compensates by learning only 1 explanation, but making that explanation as reasonable as possible.  However, the decoder can only represent uncertainty on each pixel individually, so in making a ``reasonable'' explanation, it can only make an explanation that is a blurred combination of all digits.  

In contrast, when trained with SIWAE, each of the posterior modes produces a different explanation for the data.  These different explanations correspond to different styles of each digit. This corresponds to the different styles of ones in the MNIST training set.  By allowing the posterior to provide multiple explanations, the decoder produces outputs which are less uncertain.  This not only results in improved visual appearance of the outputs, but also shows that SIWAE is able to represent more complex forms of uncertainty in the posterior predictive distribution.


In the main text of the paper, we presented several results with regard to VAE models which were given the full image as an input.  Here we will show the full details informing these results.  The first result was that in low dimensionality, SIWAE models outperformed SELBO models, and exhibited improving performance as a function of $K$ and $T$.  This is shown in \Autoref{fig:fullimVAE2}.  For SELBO, we do not find a corresponding improvment, as the $K=1$ model outperformed $K>1$.  The origins of this are unclear.  Furthermore, we find that SELBO does not appear to exhibit strong dependence on $T$.  

\section{Training procedure}\label{sec:training_supplement}
In this appendix, we describe several additions to the traditional training procedure, which we found to aid in optimization of all VAE models.

\subsection{Burn-in against the prior}
Typical initialization schemes attempt to facilitate gradient backpropagation by ensuring that the first two moments of the activations remain approximately 0, and 1, respectively.  We found that these initialization schemes don't typically produce a posterior distribution tuned reasonably to the prior.  This violates our intuition, as the posterior in the absence of evidence should be identical to the prior.  Furthermore, even if the model were initialized such that the posterior and prior were aligned, the alignment would quickly be broken by large bulk gradients being given to the posterior from the likelihood.   We observed that this was a consequence of the decoder being not well tuned to the dataset at initialization.  For example, the edge pixels in MNIST are essentially all zero, but the initial decoder predicts a uniform distribution over these pixels.  The model would therefore systematically shift the posterior to compensate for the poor initialization of the likelihood.  This bulk shift early on in training often produced a final posterior that was well tuned to the likelihood, but poorly tuned to the prior.

Our simple solution to this problem was to burn in the decoder so that the initial decoder distribution was reflective of our prior distribution over the dataset.  To do this, we fed samples from the prior to the decoder, and attempted to maximize the expected log-likelihood of the images given the prior samples $p(x|z)$.  Furthermore, because the encoder could often be improperly tuned against the prior (a worsening problem in higher dimensionality), we attempted to uniformly spread the encoder across the prior.  This was accomplished by minimizing the cross entropy between the prior and the posterior $H_{x}(r,q)=\mathbb{E}_{z\sim r(z)} q(z|x)$.  This optimization was performed jointly for both the encoder and decoder variables, with the prior held fixed.  Our specific procedure to do this was as follows.
\begin{enumerate}
    \item Draw samples from the prior distribution $r(z)$.
    \item For the prior samples, compute the expected log-likelihood $\mathbb{E}_{z_{i}\sim r(z), x_{i}; i=1...M}[\log{p(x_{i}|z_{i})}]$ over a batch of images in the training set.
    \item For the prior samples, also compute the cross entropy from the posterior $H_{x}(r(z_{i},q(z_{i}|x_{i}))$.
    \item Compute and apply the gradients of the loss  $\mathcal{L}=\mathbb{E}_{z_{i}\sim r(z), x_{i}; i=1...M}[-\log{p(x_{i}|z_{i})}-\log q(z|x)]$ for all encoder and decoder variables.
    \item Repeat until converged.  
\end{enumerate}
In practice, we found that convergence was typically achieved within a single epoch, so for simplicity we ran burn-in for a single epoch.  This produced a decoder which, when fed samples from the prior, would produce predictions consistent with random samples of each pixel from the dataset.  Note that prior samples from the burned in decoder do not resemble images from the dataset, but merely draw from a simplified estimate of the prior distribution for each pixel.  At the same time, this burn in procedure matches the encoder to the prior, which makes some sense, given that we initially only know samples from the prior.  We think that this burn in procedure is a worthwhile practice for initializing latent variable models.

\subsection{Mixture Penalties}
In addition to burning in the encoder and decoder, we also utilized several penalties to avoid common pitfalls that may occur when fitting mixture distributions.  The first was a penalty on the negative entropy of the encoder mixture distribution.
\begin{equation}
    H(\alpha_{k}) = \E_{k=1...K}[ \alpha_{k}\log{\alpha_{k}}]
\end{equation}
The entropy of the mixture distribution governs the probability of an individual mode being responsible for producing a given observation.  Because losses like SIWAE (and even SELBO to a lesser degree), use the mixture probabilities to scale gradients to the component distributions, the model can greedily and spuriously kill off a poorly initialized component.  We found that burn in largely ameliorated this issue, but kept this term in the optimization because it was a useful safeguard.  In order to maintain the validity of our evidence bound when the model was successfully using its components, we used a penalty on the entropy given by
\begin{equation}
    \mathcal{L}(H) = \operatorname{relu}(H-H_{0})
\end{equation}
where $H_{0}$ was a constant that can be chosen.  In practice, we set $H_{0}$ such that a penalty is only incurred if more than 95\% of the probability mass is concentrated in a single component.  

Our penalty on the entropy prevents no single component from amassing more than 95\% of the total component probability.  However this by itself is not sufficient to push the model to utilize every component, as individual components can still receive no probability without a penalty.  We therefore added a second penalty, similar to the first, but meant to ensure that none of the individual encoder modes are spuriously killed off.  To do this, we simply used the following penalty
\begin{equation}
    \mathcal{L}(\alpha) = \sum_{k=1}^{K}\operatorname{relu}(\alpha_{0}-\log{\alpha_{k}})
\end{equation}
where $\alpha_{0}$ is a constant which we set such that the model received a penalty for any component with less than 1\% of the probability mass.  This penalty worsens as the components become less probable, effectively preventing the model from ignoring any one component.

\subsection{Training Details}
The VAE architecture was the same architecture as used in the  \href{https://github.com/tensorflow/probability/blob/master/tensorflow_probability/examples/vae.py}{Tensorflow Probability GitHub example}, with a few small differences.  We used a mixture of Multivariate Normal distributions for the posterior, where each multivariate normal distribution had full covariance, and where the mixture probabilities were learned during training.  For the single-column MNIST VAE, we used the same MLP architecture described in \Autoref{sec:VIB} for the encoder, while the decoder used the tensorflow probability example architecture.  In order to ensure that gradients were passed to the encoder early on in training, we used a skip connection, represented as a single affine layer to project the latent space directly to the output space.  We also experimented with both the affine and tensorflow probability decoders separately, and found that the final results did not depend on the decoder architecture, though the overall performance of all models was better for the nonlinear decoder architectures.

During training of VAE models, our full loss was the sum of the SELBO or SIWAE loss and the mixture penalties:
\begin{equation}
    \mathcal{L}_{total}=\mathcal{L}_\text{SELBO/SIWAE} + \mathcal{L}_{H} + \mathcal{L}_{\alpha}
\end{equation}
We optimized this loss using the Adam optimizer with an initial learning rate of 0.0001, decayed by a factor of 0.5 every 15000 training steps and trained for a total of 100 epochs. We used a batch size of 32 examples and took $T$ samples from each of the $K$ mixture components.

\end{document}